\documentclass{article}
\usepackage[accepted]{icml2019}
\ifdefined\usebigfont

\usepackage{times}
\usepackage[fontsize=13pt]{scrextend}
\usepackage[left=1.56in,right=1.56in,top=1.74in,bottom=1.74in]{geometry}
\else
\fi

\usepackage{tcolorbox}
\usepackage{paralist}
\usepackage[shortlabels]{enumitem}
\usepackage{subcaption}
\usepackage{natbib}
\newcommand{\norm}[1]{\lVert#1\rVert}

\usepackage{wrapfig}

\RequirePackage{amsthm,amsmath,amsmath,amsfonts,amssymb}

\usepackage{graphicx}
\usepackage{color}

\newtheorem{definition}{Definition}
\newtheorem{claim}{Claim}

\newtheorem{theorem}{Theorem}
\newtheorem{lemma}{Lemma}
\newtheorem{corollary}{Corollary}
\newtheorem{assumption}{Assumption}

\newtheorem{remark}{Remark}

\renewcommand{\tilde}{\widetilde}

\newcommand{\dnote}[1]{ }
\newcommand{\inote}[1]{ }
\newcommand{\snote}[1]{ }
\newcommand{\sedit}[1]{ }
\newcommand{\nnote}[1]{ }
\newcommand{\mnote}[1]{ }
\newcommand{\jnote}[1]{ }

\newcommand{\removed}[1]{}

\DeclareMathOperator*{\argmin}{\mathrm{argmin}}

\newcommand{\ie}{i.e., }
\newcommand{\br}{{ \rho}}

\newcommand{\optTheta}{{\bt_{k}^*}}

\newcommand{\bdot}[1]{\overset{\bullet}{#1}}
\newcommand{\w}[1][]{\ifthenelse{\equal{#1}{}}{\boldsymbol{w}}{\boldsymbol{w}{(#1)}}}
\newcommand{\bt}[1][]{\ifthenelse{\equal{#1}{}}{\boldsymbol{\theta}}{\boldsymbol{w}{(#1)}}}
\newcommand{\W}[1][]{\ifthenelse{\equal{#1}{}}{\boldsymbol{W}}{\boldsymbol{W}{(#1)}}}
\newcommand{\x}[1][]{\ifthenelse{\equal{#1}{}}{\boldsymbol{x}}{\boldsymbol{x}_{#1}}}
\newcommand{\y}[1][]{\ifthenelse{\equal{#1}{}}{{y}}{{y}_{#1}}}

\newcommand{\xn}{\mathbf{x}_n}
\newcommand{\tbt}{\tilde{\bt}}

\renewcommand{\c}{\mathcal}
\renewcommand{\b}{\mathbb}
\newcommand{\bR}{\mathbb{R}}
\newcommand{\vect}[1]{\mathbf{#1}}
\newcommand{\zn}{\vect{z}_n}
\newcommand{\hp}{\alpha} 

\newcommand{\remove}[1]{{}}

\usepackage[utf8]{inputenc}
\icmltitlerunning{Lexicographic and Depth-Sensitive Margins}

\begin{document}

\twocolumn[
\icmltitle{Lexicographic and Depth-Sensitive Margins in \\ Homogeneous and Non-Homogeneous Deep Models}

\begin{icmlauthorlist}
\icmlauthor{Mor Shpigel Nacson}{technion}
\icmlauthor{Suriya Gunasekar}{ttic}
\icmlauthor{Jason D. Lee}{usc}
\icmlauthor{Nathan Srebro}{ttic}
\icmlauthor{Daniel Soudry}{technion}

\end{icmlauthorlist}

\icmlaffiliation{technion}{Technion, Israel}
\icmlaffiliation{ttic}{TTI Chicago, USA}
\icmlaffiliation{usc}{USC Los Angeles, USA}


\icmlcorrespondingauthor{Mor Shpigel Nacson}{morshpigel@google.com}


\vskip 0.3in
]



\printAffiliationsAndNotice{}  
\begin{abstract}
With an eye
toward understanding complexity control in deep learning, we study how infinitesimal regularization or gradient descent optimization lead to margin maximizing solutions in both homogeneous and {\em non homogeneous} models, extending previous work that focused on infinitesimal regularization only in homogeneous models.  To this end we study the limit of loss minimization with a diverging norm constraint (the ``constrained path''), relate it to the limit of a ``margin path'' and characterize the resulting solution.  For non-homogeneous ensemble models, which output is a sum of homogeneous sub-models, we show that this solution discards the shallowest sub-models if they are unnecessary. For homogeneous models, we show convergence to a ``lexicographic max-margin solution'', and provide conditions under which max-margin solutions are also attained as the limit of unconstrained gradient descent.
\end{abstract}

\section{Introduction \label{sec:intro}} 

Inductive bias introduced through the learning process plays a crucial role in training deep neural networks and in the generalization properties of the learned models \citep{neyshabur2014search,neyshabur2015path,zhang2017understanding,Keskar2016,neyshabur2017geometry,wilson2017marginal,Hoffer2017}. Deep neural networks used in practice are typically highly overparameterized, \ie have far more trainable parameters than training examples.
Thus, using these models, 
it is usually possible to fit the data perfectly and obtain zero training error \cite{zhang2017understanding}. 
However, simply minimizing the training loss does not guarantee good generalization to unseen data -- many global minima of the training loss indeed have very high test error \cite{Wu2017a}. 
The  inductive bias introduced in our learning process affects which specific global minimizer is chosen as the predictor. Therefore, it is essential to understand the nature of this inductive bias to understand why overparameterized models, and particularly deep neural networks, exhibit good generalization abilities.

	A common way to introduce an additional inductive bias in overparameterized models is via small amounts of regularization, or loose constraints . 
	For example, \citet{rosset2004margin,rosset2004boosting,Wei2019} show that, in overparameterized classification models, a vanishing amount of regularization, or a diverging norm constraint can lead to max-margin solutions, which in turn enjoy strong generalization guarantees. 
	
	A second and more subtle source of inductive bias is via the optimization algorithm used to minimize the underdetermined training objective \cite{gunasekar2017implicit,soudry2017implicit}. Common algorithms used in neural network training, such as stochastic gradient descent, iteratively refine the model parameters by making incremental local updates.  For different algorithms, the local updates are specified by different geometries in the space of parameters. For example, gradient descent uses an Euclidean $\ell_2$ geometry, while coordinate descent updates are specified in the $\ell_1$ geometry. The  minimizers to which such local search based optimization algorithms converge to are indeed very special and are related to the geometry of the optimization algorithm \cite{gunasekar2018characterizing} as well as the choice of model parameterization \cite{gunasekar2018implicit}.
	
	
In this work we similarly investigate the connection between margin maximization and the limits of 
\begin{itemize}
    \item The ``optimization path'' of unconstrained, unregularized gradient descent.
    \item The ``constrained path'', where we optimize with a diverging (increasingly loose) constraint on the norm of the parameters.
    \item The closely related ``regularization path'', of solutions with decreasing penalties on the norm.
\end{itemize}

To better understand the questions we tackle in this paper, and our contribution toward understanding the inductive bias introduced in training, let us briefly survey prior work.

\paragraph{Equivalence of the regularization or constrained paths and margin maximization:}

\citet{rosset2004margin,rosset2004boosting, Wei2019} investigated the connection between the regularization and constrained paths and the max-margin solution. \citet{rosset2004boosting, rosset2004margin} considered linear (hence homogeneous) models with monotone loss and explicit norm regularization or constraint, and proved convergence to the max-margin solution for certain loss functions (e.g., logistic loss) as the regularization vanishes or the norm constraint diverges. 
\citet{Wei2019} extended the regularization path result to non-linear but positive-homogeneous prediction functions,
\begin{definition}[$\hp$-positive homogeneous function]\label{def: positive homogeneous function}
A function $g(\bt):\bR^d\to\bR$ is $\hp$-positive homogeneous if  $\forall\rho>0$ and $\forall\theta\in\bR^d:\ g(\rho\bt)=\rho^\hp g(\bt)$.
\end{definition}
e.g.~as obtained by a ReLU network with uniform depth.

These results are thus limited to only positive homogeneous predictors, and  do not include deep networks with bias parameters, ensemble models with different depths, ResNets, or other models with skip connections.  Here, we extend this connection beyond positive homogeneous predictors.

Furthermore, even for homogeneous or linear predictors, there might be multiple margin maximizing solutions. For linear models, \citet{rosset2004margin} alluded to a refined set of maximum margin classifiers that in addition to maximizing the distance to the closest data point (max-margin), also maximize the distance to the second closest data point, and so on. We formulate such special maximum margin solutions as ``lexicographic max-margin'' classifiers which we introduce in Section \ref{sec: Lexicographic max margin}. We show that for general continuous homogeneous models, the constrained path with diverging norm constraint converges to these  more refined ``lexicographic max-margin'' classifiers.

\paragraph{Equivalence of the optimization path and margin maximization:}

Another line of works studied the connection between unconstrained, unregularized optimization with a specific algorithm (\ie the limit of the ``optimization path''), and the max-margin solution. For linear prediction with the logistic loss (or other exponential tail losses), we now know gradient descent \cite{soudry2017implicit,ji2018risk} as well as SGD \cite{nacson2018stochastic} converges in direction to the max-margin solution, while steepest descent with respect to an arbitrary norm converges to the max-margin w.r.t.~the corresponding norm \cite{gunasekar2018characterizing}.  All the above results are for linear prediction.  \citet{gunasekar2018implicit,nacsonconvergence2018,ji2018gradient} obtained results establishing convergence to margin maximizing solutions also for certain uniform-depth linear networks (including fully connected networks and convolutional networks), which still implement linear model.  Separately, \citet{xu2019when} analyzed a single linear unit with ReLU activation---a limited non-linear but still positive homogeneous model. Lastly, \citet{soudry2018journal} analyzed a non-linear ReLU network where only a single weight layer is optimized.

Here, we extend this relationship to general, non-linear and positive homogeneous predictors for which the loss can be minimized only at infinity. We establish a connection between the limit of unregularized unconstrained optimization and the max-margin solution.

\paragraph{Problems with finite minimizers:}
We note that the connection between regularization path and optimization path was previously considered in a different settings, where a finite (global) minimum exists. In such settings the questions asked are different than the ones we consider here, and are not about the limit of the paths. E.g., \citet{ali2018continuous} showed for gradient flow a multiplicative relation between the risk for the gradient flow optimization path and the ridge-regression regularization path. Also, \citet{suggala2018connecting} showed that for gradient flow and strongly convex and smooth loss function -- gradient descent iterates on the unregularized loss function are pointwise close to solutions of a corresponding regularized problem. 

\subsection*{Contributions}
We examine overparameterized realizable problems 
(\ie where it is possible to perfectly classify the training data), when training using monotone decreasing classification loss functions. For simplicity, we focus on the exponential loss. However, using similar techniques as in \citet{soudry2018journal} our results should extend to other exponential-tailed loss functions such as the logistic loss and its multi-class generalization.
This is indeed the common setting for deep neural networks used in practice.  

We show that in any model,
\begin{itemize}

    \item As long as the margin attainable by a (unregularized, unconstrained) model is unbounded, then the margin of the constrained path converges to the max-margin. See Corollary \ref{corollary: constrained path margin converge to max margin}.
    \item If additional conditions hold, the constrained path also converges to the ``margin path'' in parameter space (the path of minimal norm solutions attaining increasingly large margins). See section \ref{sec: Margin of Constrained path converges to maximum margin and examples}.
\end{itemize}

We then demonstrate that 
\begin{itemize}
    \item If the model is a sum of homogeneous functions of  different orders (\ie it is not homogeneous itself), then we can still characterize the asymptotic solution of both the constrained path and the margin path. See Theorem \ref{sec: Sum of  positively homogeneous functions}.
    \item This solution implies that in an ensemble of homogeneous neural networks, the ensemble will aim to discard the most \emph{shallow} network. This is in contrast to what we would expect from considerations of optimization difficulty (since deeper networks are typically harder to train \cite{he2016deep}).
    \item This also allows us to represent hard-margin SVM problems \textbf{with unregularized bias} using such models. This is in contrast to previous approaches which fail to do so, as pointed out recently \cite{nar2019crossentropy}.
\end{itemize}

Finally, for homogeneous models,
\begin{itemize}
    \item We find general conditions under which the optimization path  converges to stationary points of the margin path or the constrained path. See section \ref{sec: Optimization path converges to stationary points of the constrained path and margin path}.
    \item We show that the constrained path converges to a specific type max-margin solution, which we term the ``lexicographic max-margin''. \footnote{The authors thank Rob Shapire for the suggestion of the nomenclature during initial discussions.} See Theorem \ref{thm: Lexicographic max margin}. 
\end{itemize}

\section{Preliminaries and Basic Results}
In this paper, we will study the following exponential tailed loss function
\begin{equation}
\mathcal{L}\left(\bt\right)\triangleq\sum_{n=1}^{N}\exp\left(-f_{n}\left(\bt\right)\right)\,,
\label{eq:loss}
\end{equation}
where $f_{n}:\mathbb{R}^{d}\rightarrow\mathbb{R}$ 
is a continuous function, and $N$ is the number of samples. Also, for any norm $\Vert\!\cdot\!\Vert$ in $\mathbb{R}^{d}$ we define $\mathbb{S}^{d-1}$ as the unit norm ball in $\mathbb{R}^{d}$. 

We will use in our results the following basic lemma
\begin{lemma} \label{lem:pareto-equiv}
Let $f$ and $g$ be two functions from $\mathbb{R}^{d}$ to $\mathbb{R}$,
such that
\begin{equation}
\phi(\rho)=\min_{\mathbf{w}\in\mathbb{R}^{d}}f(\mathbf{w})\,~\mathrm{s.t.}\,\,g(\mathbf{w})\leq\rho\label{eq: min f st g}
\end{equation}
exists and is strictly monotonically decreasing in $\rho$, $\forall \rho\geq \rho_0$, for some $\rho_0$.
Then, $\forall \rho\geq \rho_0$, the optimization problem in eq. \ref{eq: min f st g} has the
same set of solutions $\left(\mathbf{w}\right)$ as 
\begin{equation}
\min_{\mathbf{w}\in\mathbb{R}^{d}}g(\mathbf{w})\,~\mathrm{s.t.}\,\,f(\mathbf{w})\leq\phi(\rho)\,,\label{eq: min g st f}
\end{equation}
whose minimum is obtained at $g(\mathbf{w})=\rho$. 
\end{lemma}
\begin{proof}
See Appendix \ref{sec: pareto-equiv proof}.
\end{proof}

\subsection{The Optimization Path}
The optimization path in the Euclidean norm $\bt\left(t\right)$, is given by
the direction of iterates of gradient descent algorithm with initialization $\bt\left(0\right)$ and learning rates $\left\{ \eta_t \right\}_{t=1}^{\infty}$,
\begin{flalign}
\nonumber &\textbf{Optimization path:}\;\;\bar{\bt}(t)=\frac{\bt(t)}{\norm{\bt(t)}},\;\text{where}&\\
&\bt\left(t\right)=\bt\left(t-1\right)-\eta_t \nabla\mathcal{L}\left(\bt\left(t-1\right)\right).&
\end{flalign} 

\subsection{The Constrained Path}
 
The constrained path for the loss in eq. \ref{eq:loss} is given by minimizer of the loss at a given norm value $\rho > 0$, \ie
\begin{flalign} \label{eq: constrained path definition}
    &\textbf{Constrained path:}\;\;\Theta_{c}\left(\rho\right)\triangleq \arg\min_{\bt \in\mathbb{S}^{d-1}}\mathcal{L}\left(\rho\bt\right)\,.&
\end{flalign}

The constrained path was previously considered for linear models \citep{rosset2004boosting}. However, most previous works (e.g. \citet{rosset2004margin, Wei2019}) focused on the regularization path, which is the minimizer of the regularized loss. These two paths are closely linked,
as we discuss in more detail in Appendix~\ref{sec: Regularization Path}.

Denote the constrained minimum of the loss as follows:
\[
\mathcal{L}^{*}\left(\rho\right)\triangleq\min_{\bt \in\mathbb{S}^{d-1}}\mathcal{L}\left(\rho\bt\right).
\]
$\mathcal{L}^{*}\left(\rho\right)$ exists for any finite $\rho$ as the minimum of a continuous function
on a compact set.

By Lemma \ref{lem:pareto-equiv}, the Assumption \begin{assumption} \label{assumption: L^* is strictly decreasing}
    There exists $\rho_{0}$ such that $\mathcal{L}^{*}\left(\rho\right)$
    is strictly monotonically decreasing to zero for any $\rho\geq\rho_{0}$.
\end{assumption}enables an alternative form of the constrained path
\[
\Theta_{c}\left(\rho\right)=\arg\min_{\bt \in\mathbb{R}^{d}}\left\Vert \bt\right\Vert ^{2}\,\,\mathrm{s.t.\,\,}\mathcal{L}\left(\rho\bt\right)\leq\mathcal{L}^{*}\left(\rho\right)\,.
\]
In addition, in the next lemma we show that under this assumption the constrained path minimizers are obtained on the boundary of $\mathbb{S}^{d-1}$.
\begin{lemma}\label{lem: constrained path optimal solution norm is 1}
Under assumption \ref{assumption: L^* is strictly decreasing}, for all $\rho>\rho_0$ and for all $\theta_{c}\in\Theta_{c}\left(\rho\right)$, we have $\left\Vert \bt_{c}\right\Vert =1$.
\end{lemma}
\begin{proof}
Let $\rho>0$. We assume, in contradiction, that $\exists \theta_{c}\in\Theta_{c}\left(\rho\right)$ so that $\left\Vert \bt_{c}\right\Vert = b<1$. This implies that $\mathcal{L}^*\left(\rho\right)=\mathcal{L}^*\left(\rho b\right)$ which contradicts our assumption that $\mathcal{L}^*\left(\rho\right)$ is strictly monotonically decreasing.
\remove{Note that $\forall\rho>0$: $\min_{\bt:\left\Vert \bt\right\Vert \leq1}\mathcal{L}\left(\rho\bt\right)$
is attained, \ie the set $\Theta_{c}\left(\rho\right)$ is not empty.
Additionally, $\min_{\bt:\left\Vert \bt\right\Vert \leq1}\mathcal{L}\left(\rho\bt\right)$
is equivalent to $\min_{\bt:\left\Vert \bt\right\Vert ^{2}\leq1}\mathcal{L}\left(\rho\bt\right)$.
The KKT condition for the latter problem are $\exists\lambda\ge0$
so that
\begin{align*}
\nabla_{\bt}\mathcal{L}\left(\rho\bt\right)+2\lambda\bt & =0\,,\\
\lambda\left(\left\Vert \bt\right\Vert ^{2}-1\right) & =0\,.
\end{align*}

$\forall\bt_{c}\in\Theta_{c}\left(\rho\right)$, $\bt_{c}$
must satisfy these conditions (these are necessary conditions for
this problem). Since $\mathcal{L}\left(\rho\bt\right)$ does not
have any finite critical points\dnote{why?}, we must have that $\lambda>0$ and
thus $\left\Vert \bt\right\Vert ^{2}=1\Rightarrow\left\Vert \bt\right\Vert =1$.}
\end{proof}

\subsection{The Margin Path}
For prediction functions $f_{n}:\bR^d\to\bR$ on data points indexed as $n=1,2,\ldots,N$, we define the margin path as:
\begin{flalign}  \label{eq:margin-path basic definition}
 \textbf{Margin path:}\;\;   &\Theta_{m}\left(\rho\right)  \triangleq\arg\max_{\bt \in\mathbb{S}^{d-1}}\min_{n}f_{n}\left(\rho\bt\right)\,.&
\end{flalign}

For $\bt\in\mathbb{S}^{d-1}$, we denote the margin at scaling $\rho>0$ as 
\[
\gamma\left(\rho,\bt\right)=\min_{n}f_{n}\left(\rho\bt\right)\,,
\]
and the max-margin at scale of $\rho>0$ as
\[
\gamma^{*}\left(\rho\right)=\max_{\bt \in\mathbb{S}^{d-1}}\min_{n}f_{n}\left(\rho\bt\right)~.
\]
Note that for all $\rho$, this maximum exists as the maximum of a continuous function on a compact set.

Again, we make a simplifying assumption 
\begin{assumption} \label{assum: monotone gamma}
    There exist $\rho_0$ such that $\gamma^{*}\left(\rho\right)$ is strictly monotonically increasing to $\infty$ for any $\rho\geq\rho_{0}$.
\end{assumption}
Many common prediction functions satisfy this assumption, including the sum of positive-homogeneous prediction functions.

Using Lemma~\ref{lem:pareto-equiv} with Assumption~\ref{assum: monotone gamma}, we have:
\begin{align} \label{eq:margin-path}
    \Theta_{m}\left(\rho\right) & =\arg\max_{\bt \in\mathbb{S}^{d-1}}\min_{n}f_{n}\left(\rho\bt\right) \\
    & =\arg\min_{\bt \in\mathbb{R}^{d}}\left\Vert \bt\right\Vert ^{2}\,\,\mathrm{s.t.}\,\,\min_{n}f_{n}\left(\rho\bt\right)\geq\gamma^{*}\left(\rho\right)\,. \nonumber
\end{align}

\section{Non-Homogeneous Models} \label{sec: Non-Homogeneous Models}

We first study the constrained path in non-homogeneous models, and relate it to the margin path. To do so, we need to first define  
the $\epsilon$-ball surrounding a set $\mathcal{A}\subset \mathbb{R}^d$ $$\mathcal{B}_{\epsilon}\left(\mathcal{A}\right)\triangleq\left\{ \bt\in\mathbb{R}^{d}~|~\exists\bt'\in\mathcal{A}:\left\Vert \bt-\bt'\right\Vert <\epsilon\right\}\,, $$
and the notion of set convergence
\begin{definition}[Set convergence]
A sequence of sets   $\mathcal{A}_{t}\subset\mathbb{R}^{d}$ converges
to another sequence of sets $\mathcal{A}_{t}^{\prime}\subset\mathbb{R}^{d}$
if $\forall\epsilon>0$ $\exists t_{0}$ such that $\forall t>t_{0}$
$\mathcal{A}_{t}\subset\mathcal{B}_{\epsilon}\left(\mathcal{A}_{t}^{\prime}\right)$.   
\end{definition}

\subsection{Margin of Constrained Path Converges to Maximum Margin} \label{sec: Margin of Constrained path converges to maximum margin and examples}
For all $\rho$, the constrained path margin deviation from the max-margin is bounded, as we prove next.
\begin{lemma} \label{lem: difference between constrained path margin and max margin is bounded}
For all $\rho$, and every $\bt_{c}\left(\rho\right)$ in $\Theta_{c}\left(\rho\right)$
\begin{equation} \label{eq: constrained margin deviation from max margin bound}
    \gamma^{*}\left(\rho\right)-\gamma\left(\rho,\bt_{c}\left(\rho\right)\right)\leq \log N\,.
\end{equation}
\end{lemma}
\begin{proof}
    Note that $\forall\bt:$
    \begin{align}
    e^{-\gamma\left(\rho,\bt\right)}\leq\sum_{n=1}^{N}\exp\left(-f_{n}\left(\rho\bt\right)\right)\leq N e^{-\gamma\left(\rho,\bt\right)}.    
    \end{align}
    Since, $\forall\bt\in\mathbb{S}^{d-1}\, , \mathcal{L}\left(\rho\bt_{c}\left(\rho\right)\right)\leq\mathcal{L}\left(\rho\bt\right)$,
    we have,  $\forall\bt_{c}\left(\rho\right)\in \Theta_{c}\left(\rho\right)$
    and $\forall\bt_{m}\left(\rho\right)\in\Theta_{m}\left(\rho\right)$,
    \begin{align*}
    1 &\leq\frac{\mathcal{L}\left(\rho\bt_{m}\left(\rho\right)\right)}{\mathcal{L}\left(\rho\bt_{c}\left(\rho\right)\right)}=\frac{\sum_{n=1}^{N}\exp\left(-f_{n}\left(\rho\bt_{m}\left(\rho\right)\right)\right)}{\sum_{n=1}^{N}\exp\left(-f_{n}\left(\rho\bt_{c}\left(\rho\right)\right)\right)} \\ 
    &\leq N\exp\left(-\left(\gamma^{*}\left(\rho\right)-\gamma\left(\rho,\bt_{c}\left(\rho\right)\right)\right)\right)\,.\label{eq: Loss ratio}
    \end{align*}
    \[\Rightarrow \gamma^{*}\left(\rho\right)-\gamma\left(\rho,\bt_{c}\left(\rho\right)\right)\leq \log N\,. \qedhere\]
\end{proof}

Lemma \ref{lem: difference between constrained path margin and max margin is bounded} immediately implies that
\begin{corollary} \label{corollary: constrained path margin converge to max margin}
If $\lim_{\rho\rightarrow\infty}\gamma^{*}\left(\rho\right)=\infty$,
then for all $\rho$, and every $\bt_{c}\left(\rho\right)$ in $\Theta_{c}\left(\rho\right)$
\[
\lim_{\rho\rightarrow\infty}\frac{\gamma^{*}\left(\rho\right)}{\gamma\left(\rho,\bt_{c}\left(\rho\right)\right)}=1\,.
\]
\end{corollary}

The last corollary states that the margin of the constrained path converges to the maximum margin. However, this does not necessarily imply convergence in parameter space, \ie this result does not guaranty that $\Theta_c\left(\rho\right)$ converges to $\Theta_m\left(\rho\right)$. We analyze some positive and negative examples to demonstrate this claim.

\textbf{Example 1: homogeneous models}


It is straightforward to see that, for $\hp$-positive homogeneous prediction functions (Definition \ref{def: positive homogeneous function}) the margin path $\Theta_m(\rho)$ in eq. \ref{eq:margin-path basic definition} is the same set for any $\rho$, and is given by 
\[\Theta_m^*=\arg \max_{\bt\in\b{S}^{d-1}}\min_{n} f_n\left(\bt\right)\,. \] 
Additionally, as we show next, for such models Lemma~\ref{lem: difference between constrained path margin and max margin is bounded} implies convergence in parameter space,
\ie $\Theta_{c}\left(\rho\right)$ converges to
$\Theta_{m}\left(\rho\right)$.
To see this, notice that for $\hp$-positive homogeneous functions $f_n$, $\forall \bt_{c}\left(\rho\right) \in\Theta_{c}\left(\rho\right)$:
\begin{align*}
    &\gamma^*\left(\rho\right) - \gamma\left(\rho,\bt_{c}\left(\rho\right)\right) =\\
    &\max_{\bt \in\mathbb{S}^{d-1}}\min_{n}f_{n}\left(\rho\bt\right) - \min_{n}f_{n}\left(\rho\bt_{c}\left(\rho\right)\right)\\
    & \rho^{\hp} \left( \max_{\bt \in\mathbb{S}^{d-1}}\min_{n}f_{n}\left(\bt\right) - \min_{n}f_{n}\left(\bt_{c}\left(\rho\right)\right)\right)\\
    &\le \log N\,.
\end{align*}
For $\rho\to\infty$ we must have \[\left( \max_{\bt \in\mathbb{S}^{d-1}}\min_{n}f_{n}\left(\bt\right) - \min_{n}f_{n}\left(\bt_{c}\left(\rho\right)\right)\right) \to 0\,.\] 
By continuity, the last equation implies that $\Theta_{c}\left(\rho\right)$ converges to
$\Theta_{m}\left(\rho\right)$. For full details see Appendix~\ref{sec: why margin convergence implies parameters convergence in homogeneous models}.

\textbf{Connection to previous results:} 
 For linear models,  \citet{rosset2004boosting} connected the $L_1$ constrained path and maximum $L_1$ margin solution. In addition, for any norm, \citet{rosset2004margin} showed that the regularization path converges to the limit of the margin path. 
In a recent work, 
\citet{Wei2019} extended this result to homogeneous models with cross-entropy loss. Here, for homogeneous models and any norm, we show a connection between the constrained path and the margin path. 

\textbf{Extension:}
Later, in Theorem~\ref{thm: Lexicographic max margin} we prove a more refined result: the constrained path converges to a \emph{specific subset} of the margin path set (the lexicographic max-margin set). 

In contrast, in general models, \ref{eq: constrained margin deviation from max margin bound} does not necessarily imply convergence in the parameter space. We demonstrate this result in the next example.

\textbf{Example 2: log predictor:}
We denote $\zn=y_n\xn$ for some dataset $\left\{\xn,y_n\right\}_{n=1}^N$, with features $\xn$ and label $y_n$. We examine the prediction function $f_{n}\left(\rho,\bt\right)=\log\left(\rho\bt^{\top}\zn\right)$ for $\bt^{\top}\zn>0$. We focus on the loss function tail behaviour and thus only care about the loss function behaviour in $\bt^{\top}\zn>0$ region. We assume that a separator which satisfy this constraint exists since we are focusing on  realizable  problems. 

Since  $\log(.)$  is strictly increasing and $\rho>0$, we have 
\[\gamma\left(\rho,\bt\right)=\min_{n}f_{n}\left(\rho\bt\right)=\log\left(\rho\min\limits_{n}\bt^{\top}\zn\right)\,.\]
We denote $\tilde{\gamma}\left(\bt\right)=\min\limits_n \bt^\top\zn$ and $\tilde{\gamma}^*=\max\limits_{\bt\in\mathbb{S}^{d-1}} \tilde{\gamma}\left(\bt\right)$. Note that $\gamma^*\left(\rho\right)=\log\left(\rho\tilde{\gamma}^*\right)$.
Now consider 
$\bt_c\left(\rho\right)\in\Theta_c\left(\rho\right)$ such that for some $\rho_0$ and $\forall \rho>\rho_0$:  $\tilde{\gamma}\left(\bt_c\left(\rho\right)\right)=\min\limits_{n} \bt_c\left(\rho\right)^\top\zn\ge \frac{\tilde{\gamma}^*}{N}$.
For this case, we still have, 
\begin{align*}
\gamma^{*}\left(\rho\right) - \gamma\left(\rho,\bt_c\left(\rho\right)\right)
&= \log\left(\rho\tilde{\gamma}^*\right) - \log\left(\rho\tilde{\gamma}\left(\bt_c\left(\rho\right)\right)\right)\\
&=\log\left(\frac{\tilde{\gamma}^*}{\tilde{\gamma}\left(\bt_c\left(\rho\right)\right)}\right)\le \log N\,.
\end{align*}
but clearly,  $\tilde{\gamma}\left(\bt_c\left(\rho\right)\right)\nrightarrow\tilde{\gamma}^*$. 
Thus, Lemma \ref{lem: difference between constrained path margin and max margin is bounded} does not guarantee that $\tilde{\gamma}\left(\bt_c\left(\rho\right)\right)\to\tilde{\gamma}^*$ as $\rho\to\infty$, or that $\Theta_c\left(\rho\right)$ converges to $\Theta_m\left(\rho\right)$.

\textbf{Analogies with regularization and optimization paths:}
This example demonstrates that for the prediction function $\log(\rho\bt^{\top}\mathbf{z})$ for $\bt^{\top}\mathbf{z}>0$, the constrained path does not necessarily converge to the margin path. This is equivalent to \textit{setup A}: linear prediction models with loss function $\exp\left(-\log\left(u\right)\right)$. \citet{rosset2004margin} and \citet{nacsonconvergence2018} state related results for \textit{setup A}.
Both works derived conditions on the loss function that ensure convergence to the margin path from the regularization/ optimization path respectively. \citet{rosset2004margin} showed that in \textit{setup A} the regularization path does not necessarily converge to the margin path.
\cite{nacsonconvergence2018} showed a similar result for the optimization path, \ie that in \textit{setup A} the optimization path does not necessarily converge to the margin path. Both results align with our results for the constrained path.

In contrast, according to the conditions of \citet{rosset2004margin, nacsonconvergence2018}, we know that if the prediction function is $\log^{1+\epsilon}(\rho\bt^{\top}\mathbf{z})$ for some $\epsilon>0$ and $\bt^{\top}\mathbf{z}>0$, then the regularization path and optimization path \emph{do converge} to the margin path. 
In the next example, we show that this is also true for the constrained path.

\textbf{Example 3: $(1+\epsilon)$-log predictor}: We examine the  prediction function $f_{n}\left(\rho,\bt\right)=\log^{1+\epsilon}\left(\rho\bt^{\top}\zn\right)$  for $\bt^{\top}\zn>0$ and some $\epsilon>0$. Since the log function is strictly increasing and $\epsilon,\rho>0$, we have 
\[\gamma\left(\rho,\bt\right)=\min_{n}f_{n}\left(\rho\bt\right)=\log^{1+\epsilon}\left(\rho\min\limits_{n}\bt^{\top}\zn\right)\,.\]
For all 
$\bt_c\left(\rho\right)\in\Theta_c\left(\rho\right)$:
\begin{align*}
&\gamma^{*}\left(\rho\right) - \gamma\left(\rho,\bt_c\left(\rho\right)\right)\\
&  =(1+\epsilon)\log^\epsilon \left(\rho\right)\left(\log\left(\tilde{\gamma}^*\right) 
- \log\left(\tilde{\gamma}\left(\bt_c\left(\rho\right)\right)\right) \right) \\
&\quad+o\left(\log^\epsilon \left(\rho\right)\right) \le N\,.
\end{align*}
For $\rho\to\infty$ we must have $\left(\log\left(\tilde{\gamma}^*\right) 
- \log\left(\tilde{\gamma}\left(\bt_c\left(\rho\right)\right)\right) \right)\to0$, which implies, by continuity, that $\Theta_{c}\left(\rho\right)$ converges to
$\Theta_{m}\left(\rho\right)$. For details, see Appendix \ref{sec: Auxiliary calculation for log^1+epsilon example}.

\subsection{Sum of  Positively Homogeneous Functions} \label{sec: Sum of  positively homogeneous functions}
\textbf{Remark:} The results in this subsection are specific for the Euclidean or $L_2$ norm.

Let $f_{n}\left(\br \bt\right)$ be functions that are a finite sum of positively homogeneous functions, \ie  for some finite $K$:
\begin{equation} \label{eq:non-homo}
    \forall n:\,f_{n}\left(\br \bt\right)=\sum_{k=1}^{K}f_{n}^{\left(k\right)}\left(\br\bt_{k}\right)\,,
\end{equation}
where $\bt=\left[\bt_{1},\dots,\bt_{K}\right]$ and $f_{n}^{\left(k\right)}\left(\bt_{k}\right)$
are $\hp_{k}$-positive homogeneous functions, where $0<\hp_{1}<\hp_{2}<\dots<\hp_{K}$.

First, we characterize the asymptotic form of the margin path in this setting.

\begin{lemma} \label{thm:sum-margin-path}
Let  $f_{n}\left(\bt\right)$ be a sum of positively homogeneous functions as in eq. \ref{eq:non-homo}. 
Then, the set of solutions of 
\begin{equation}
\arg\min_{\bt\in\mathbb{R}^d}\left\Vert \bt\right\Vert ^{2}\mathrm{\,\,s.t.\,\,}\forall n:f_{n}\left(\br \bt\right)\geq\gamma^*\left(\rho\right)\,.\label{eq: margin path}
\end{equation} can be written as 
\begin{equation}
\optTheta=\frac{1}{\br}\left(\w_{k}+o\left(1\right)\right)\left(\gamma^*\left(\rho\right)\right)^{\frac{1}{\hp_{k}}}\label{eq: theta_k}
\end{equation}
where the $o\left(1\right)$ term is vanishing as $\gamma^*\left(\rho\right)\rightarrow\infty$,
and $$\w^*=\left[\w_{1}^*,\dots,\w_{K}^*\right]\in\mathcal{W},$$ where 
\begin{align}
\mathcal{W} =\arg\min_{\w\in\mathbb{R}^d}\left\Vert \w_{1}\right\Vert ^{2}\mathrm{\,\,s.t.\,\,}\forall n:f_{n}\left(\w\right)\geq1 \text{.}\label{eq: infinite gamma}
\end{align}
\end{lemma}

\begin{proof}
We write the original optimization problem
\[
\arg\min_{\bt\in\mathbb{R}^d}\sum_{k=1}^{K}\left\Vert \bt_{k}\right\Vert ^{2}\mathrm{\,\,s.t.\,\,}\forall n:\sum_{k=1}^{K}f_{n}^{\left(k\right)}\left(\br \bt_{k}\right)\geq\gamma^*\left(\rho\right)\,.
\]
Dividing by $\gamma^*\left(\rho\right)$, using the $\hp_{k}$ positive homogeneity
of $f_{n}^{\left(k\right)}$, and changing the variables as $\bt_{k}=\frac{1}{\br}\w_{k}\left(\gamma^*\left(\rho\right)\right)^{\frac{1}{\hp_{k}}}$, we obtain an equivalent optimization problem
\begin{equation}
\arg\min_{\w\in\mathbb{R}^d}\sum_{k=1}^{K}\gamma^*\left(\rho\right)^{\frac{2}{\hp_{k}}}\left\Vert \w_{k}\right\Vert ^{2}\mathrm{\,\,s.t.\,\,}\forall n:f_{n}\left(\w\right)\geq1\,.\label{eq: finite gama}
\end{equation}
We denote the set of solutions of eq. \ref{eq: finite gama} as $\mathcal{W}\left(\gamma^*\left(\rho\right)\right)$. 
Taking the limit of $\gamma^*\left(\rho\right)\rightarrow\infty$ of this optimization
problem we find that any solution $\w\in \mathcal{W}\left(\gamma^*\left(\rho\right)\right)$ must minimize the first term
in the sum $\left\Vert \w_{1}\right\Vert ^{2}$, and only then the other terms.
Therefore the asymptotic solution is of the form of eqs. \ref{eq: theta_k}  and \ref{eq: infinite gamma}. We prove this reasoning formally in Appendix~\ref{sec: aux proof 1}, \ie we show that

\begin{claim} \label{claim: aux claim 1}
The solution of eq. \ref{eq: finite gama} is the same solution described in Lemma \ref{thm:sum-margin-path}, \ie eqs. \ref{eq: theta_k} and \ref{eq: infinite gamma}.
\end{claim}
\vspace{-0.6cm}
\end{proof}
\vspace{-0.2cm}
The following Lemma will be used to connect the constrained path to the characterization of the margin path.

\begin{lemma} \label{lem:homo-sum-constrained-margin}
Let  $f_{n}\left(\br \bt\right)$ be a sum of positively homogeneous functions as in eq. \ref{eq:non-homo}. 
Any path $\bt\left(\rho\right)$ such that 
\begin{equation}
\gamma^{*}\left(\rho\right)-\gamma\left(\rho,\bt\left(\rho\right)\right)<C\,.\label{eq: gamma difference}
\end{equation}
is of the form described in eqs. \ref{eq: theta_k}  and \ref{eq: infinite gamma}.
\end{lemma}
\begin{proof}
See Appendix~\ref{sec:homo-sum-constrained-margin-proof}.
\end{proof}

Combining Lemma \ref{lem: difference between constrained path margin and max margin is bounded}, \ref{thm:sum-margin-path} and Lemma \ref{lem:homo-sum-constrained-margin} we obtain the following Theorem

\begin{theorem} \label{thm:hom-sum-final}
    Under Assumption \ref{assumption: L^* is strictly decreasing} and \ref{assum: monotone gamma}, any solution in 
    $\argmin_{\|\bt\|\le 1} \mathcal{L(\br\bt)}$
    converges to 
    \begin{equation}
    \optTheta=\frac{1}{\br}\left(\w_{k}^*+o\left(1\right)\right)(\gamma^*(\rho))^{\frac{1}{\hp_{k}}}
    \end{equation}
where the $o\left(1\right)$ term is vanishing as $\gamma^*(\rho)\rightarrow\infty$,
and $$\w^*=\left[\w_{1}^*,\dots,\w_{K}^*\right]\in\mathcal{W},$$ where 
\begin{align}
\mathcal{W} =\arg\min_{\w\in\mathbb{R}^d}\left\Vert \w_{1}\right\Vert ^{2}\mathrm{\,\,s.t.\,\,}\forall n:f_{n}\left(\w\right)\geq1 \text{.} 
\end{align}
    
\end{theorem}

\textbf{Theorem \ref{thm:hom-sum-final} implications:}
An important implication of Theorem \ref{thm:hom-sum-final} is that an ensemble on neural networks will aim to discard the shallowest network in the ensemble. Consider the following setting: for each $k\in\left\{1,\dots,K\right\}$, the function  $\forall n:\, f_{n}^{\left(k\right)}\left(\br\bt_{k}\right)$ represents a prediction function of some feedforward neural network with no bias, 
all with the same positive-homogeneous activation function $\sigma\left(\cdot\right)$ of some degree $\hp$ (\textit{e.g.}, ReLU activation is positive-homogeneous of degree $1$). Note that in this setup, each of the $k$ prediction functions $f_{n}^{\left(k\right)}\left(\br\bt_{k}\right)$ is also a positive-homogeneous function. In particular, network $k$ with depth $d_k$ is positive homogeneous with degree $\hp_k=\hp{d_k}$ where $\hp$ is the activation function degree. Since  all the networks have the same activation function, deeper networks will have larger degree. We assume WLOG that $d_1<d_2<\dots<d_K$. This implies that $\hp_1<\hp_2<\dots<\hp_K$. In this setting, $\forall n:\,f_{n}\left(\br\bt\right)=\sum_{k=1}^{K}f_{n}^{\left(k\right)}\left(\br\bt_{k}\right)$ represents an ensemble of these networks. From Theorem \ref{thm:hom-sum-final}, the solution of the constrained path will satisfy
\vspace{-0.3em}
\begin{align*}
    f_{n}\left(\br\bt^*\right)&=\sum_{k=1}^{K}f_{n}^{\left(k\right)}\left(\br\bt_{k}^*\right)\\
    &= \sum_{k=1}^{K}f_{n}^{\left(k\right)}\left(\left(\w_{k}^*+o\left(1\right)\right)\left(\gamma^*(\br)\right)^{\frac{1}{\hp_{k}}}\right)\\
    & = \gamma^*(
    \rho)\sum_{k=1}^{K}f_{n}^{\left(k\right)}\left(\w_{k}^*+o\left(1\right)\right)\,,
\end{align*}
where $\w^*\in\mathcal{W}$ and $\mathcal{W}$ is calculated using eq. \ref{eq: infinite gamma}. Examining equation \ref{eq: infinite gamma}, we observe that the network aims to minimize the $\w_{1}$ norm. In particular, if the network ensemble can satisfy the constraints $\forall n:f_{n}\left(\w\right)\geq1$ with $\w_{1}=\mathbf{0}$, then the first equation obtained solutions will satisfy $\w_{1}=\mathbf{0}$. Thus the ensemble will discard the shallowest network if it is "unnecessary" to satisfy the constraint. 

Furthermore, from eq. \ref{eq: finite gama} we conjecture that after discarding the shallowest ``unnecessary" network, the ensemble will tend to minimize $\norm{\w_{2}}$, \ie to discard the second shallowest "unnecessary" network. This will continue until there are no more "unnecessary" shallow networks. In other words, we conjecture that the an ensemble of neural networks will aim to discard the shallowest ``unnecessary" networks. 

Additionally, using Theorem \ref{thm:hom-sum-final} we can now represent  hard-margin  SVM problems
\textbf{with unregularized bias}.
Previous results only focused on linear prediction functions without bias. Trying to extend these results to SVM with bias by extending all the input vectors $\xn$ with an additional $'1'$ component would fail since the obtained solution in the original $\mathbf{x}$ space is the solution of
$$\underset{\mathbf{\mathbf{w}}\in\mathbb{R}^{d},b\in\mathbb{R}}{\mathrm{argmin}}\left\lVert \mathbf{w}\right\rVert ^{2}+b^2\,\,\mathrm{s.t.}\,\,y_n\left(\mathbf{w}^{\top}\mathbf{x}_{n}+b\right)\geq1, $$
which is not the $L_2$ max-margin (SVM) solution, as pointed out by \cite{nar2019crossentropy}.
However, we can now achieve this goal using Theorem \ref{thm:hom-sum-final}. For some dataset $\left\{ \xn ,y_n\right\}_{n=1}^N,\,\xn\in\mathbb{R}^d,\ y_n\in\{-1,1\}$, we use the following prediction function $f_n\left(\bt\right) = y_n\left(\bt_1^\top\xn + b^2\right)$ where $\bt = \left[ \bt_1,b \right]$. From eqs. \ref{eq: theta_k}, \ref{eq: infinite gamma} the asymptotic solution will satisfy $\arg\min\limits_{\bt_1,b}\left\Vert \bt_{1}\right\Vert ^{2}\mathrm{\,\,s.t.\,\,}\forall n:y_n\left(\bt_1^\top\xn+b^2\right)\geq1$. 

\section{Homogeneous Models}
\label{sec:homogeneous}
In the previous section we connected the constrained path to the margin path. We would like to refine this characterization and also understand the connection to the optimization path. In this section we are able to do so for prediction functions $f_n\left(\bt\right)$ which are $\hp$-positive homogeneous functions (definition \ref{def: positive homogeneous function}).

In the homogeneous case, eq.  \ref{eq:margin-path} is equivalent, $\forall \rho$, to
\begin{align}
\Theta^*_{m}= \arg\min_{\bt \in\mathbb{R}^{d}}\left\Vert \bt\right\Vert ^{2}\,\,\mathrm{s.t.}\,\,\min_{n}f_{n}(\bt) \ge \gamma^* (1) \label{eq:margin-homogeneous}
\end{align}
since $f_n$ is homogeneous. 

\subsection{Optimization Path Converges to Stationary Points of the Margin Path and Constrained Path} \label{sec: Optimization path converges to stationary points of the constrained path and margin path}
\textbf{Remark:} The results in this subsection are specific for the Euclidean or $L_2$ norm, as opposed to many of the results in this paper which are stated for any norm.

In this section, we link the optimization path to the margin path and the constrained path. These results require the following smoothness assumption:
\begin{assumption}[Smoothness]
	We assume $f_n (\cdot)$ is a $\mathcal{C}^2$ function.
	\label{assu:diff}
\end{assumption}
\vspace{-1.1em}

\paragraph{Relating optimization path and margin path.} The limit of the margin path for homogeneous models is given by eq. \ref{eq:margin-homogeneous}. In this section we first relate the optimization path to this limit of margin path.

Note that for general homogeneous prediction functions $f_n$, eq. \ref{eq:margin-homogeneous} is a non-convex optimization problem, and thus it is unlikely for an optimization algorithm such as gradient descent to find the global optimum. We can relax the set to $\bt$ that are first-order stationary, \ie critical points of \ref{eq:margin-homogeneous}.
For $\bt\in\b{S}^{d-1}$, denote the set of support vectors of $\bt$ as 
\begin{equation}
S_m(\bt) =\{n: f_n (\bt ) =\gamma^*(1)\}\,.
\label{eq:sv}
\end{equation} 
\begin{definition}[First-order Stationary Point]
	The first-order optimality conditions of \ref{eq:margin-homogeneous} are:	
	\begin{enumerate}[wide, labelindent=0pt,topsep=0pt]
		\item $\forall n$, $f_n (\bt) \ge \gamma^* (1)$
		\item There exists $\boldsymbol{\lambda} \in \mathbb{R}^N_+$ such that $\bt = \sum_{n } \lambda_n \nabla f_n (\bt) $ and $\lambda_n =0 $ for $n \notin S_m(\bt)$ .
	\end{enumerate}
	We denote by $\Theta_m ^s $ the set of first-order stationary points.
	\label{def:margin-first-order}
\end{definition}

Let $\bt(t)$ be the iterates of gradient descent. Define $\ell_{n} (t) = \exp( - f_n (\bt(t)))$ and $\boldsymbol{\ell}(t)$ be the vector with entries $\ell_n (t)$. The following two assumptions assume that the limiting direction $\frac{\bt(t) }{\|\bt(t)\|}$ exist and the limiting direction of the losses $\frac{\boldsymbol{\ell}(t)}{\|\boldsymbol{\ell}(t)\|_1}$ exist. Such assumptions are natural in the context of max-margin problems, since we want to argue that $\bt(t)$ converges to a max-margin direction, and also the losses $\boldsymbol{\ell}(t)/ \|\boldsymbol{\ell}(t)\|_1$ converges to an indicator vector of the support vectors. The first step to argue this convergence is to ensure the limits exist.
\begin{assumption}[Asymptotic Formulas]
	Assume that $\mathcal{L}(\bt(t)) \to 0$, that is we converge to a global minimizer. Further assume that $\lim\limits_{t\to\infty} \frac{\bt(t)}{\norm{\bt(t)}_2}$ and $\lim\limits_{t\to\infty} \frac{\boldsymbol{\ell} (t) }{\norm{\boldsymbol{\ell}(t)}_1}$ exist. Equivalently,
	\begin{align}
	\ell_n (t) &= h(t) a_n + h(t) \epsilon_n (t)\\
	\bt (t) &= g(t) \bar \bt +g(t) \boldsymbol{\delta}(t),
	\end{align}
	with $\norm{\textbf{a}}_1 =1$, $\norm{\bar \bt}_2 =1$, $\lim\limits_{t\to\infty} h(t) =0$, $\lim\limits_{t\to\infty} \epsilon_n (t) =0$, and $\lim\limits_{t\to\infty} \boldsymbol{\delta}(t) =0$. 
	\label{assu:asymptotic-opt}
\end{assumption}
\begin{assumption}[Linear Independence Constraint Qualification]
	Let $\bt\in \b{S}^{d-1}$ be a unit vector. LICQ holds at $\bt$ if the vectors $\{\nabla f_n (\bt) \}_{n \in S_m(\bt)}$ are linearly independent.
	\label{assu:cq}
\end{assumption}

\begin{remark}
	Constraint qualifications allow the first-order optimality conditions of Definition \ref{def:margin-first-order} to be a necessary condition for optimality. Without constraint qualifications, the global optimum need not satisfy the optimality conditions. 
	
	LICQ is the simplest among many constraint qualification conditions identified in the optimization literature \cite{wright1999numerical}. 
	
For example, in linear SVM, LICQ is ensured if the set of support vectors is linearly independent. Consider $f_n (\bt) = \mathbf{x}_n ^\top \bt$ and $\xn$ be the support vectors. Then $\nabla f_n (\bar \bt) = \mathbf{x}_n$ , and so linear independence of the support vectors implies LICQ. For data sampled from an absolutely continuous distribution, the SVM solution will always have linearly independent support vectors \citep[Lemma 12]{soudry2017implicit}, but LICQ may fail when the data is degenerate.
\end{remark}
\begin{theorem}
	
	Define $\bar \bt = \lim\limits_{t\rightarrow\infty} \frac{\bt(t)}{\norm{\bt(t)}_2}$. Under Assumptions \ref{assu:diff}, \ref{assu:asymptotic-opt}, and constraint qualification at $\bar \bt$ (Assumption \ref{assu:cq}), $\bar \bt$ is a first-order stationary point of \ref{eq:margin-homogeneous}.
	\label{thm:margin-opt-homogeneous}
\end{theorem}
The proof of Theorem \ref{thm:margin-opt-homogeneous} can be found in Appendix \ref{sec: margin-opt-homogeneous proof}.

\paragraph{Optimization path and constrained path.} Next, we study how the optimization path as $t\to \infty$ converges to stationary points of the constrained path with  $\rho \to \infty$.

The first-order optimally conditions of the constrained path $\min_{\|\bt\|\le 1} \mathcal{L}(\rho\bt)$, require that the constraints hold, and the gradient of the Lagrangian of the constrained path 
\begin{equation} \label{eq: lagrangian gradient}
    \rho\nabla_{\bt} \mathcal{L}(\rho\bt)+\lambda(\rho) \bt
\end{equation}
is equal to zero. In other words,
\begin{remark}
 	Under Assumption \ref{assumption: L^* is strictly decreasing}, $\bt$ is first-order optimal for the problem $\min_{\|\bt\|\le 1} \mathcal{L}(\rho\bt) $ if it satisfies:
 	
	\begin{inparaitem}
		\item $-\frac{ \nabla_{\theta} \mathcal{L}( \rho\bt)}{\|\nabla_{\theta} \mathcal{L}(\rho\bt)\|}= \frac{\bt}{\|\bt\|}$\,,\qquad
		\item $\| \bt \| = 1$.
	\end{inparaitem} 	
\end{remark}

On many paths the gradient of the Lagrangian goes to zero as $\rho\rightarrow\infty$.
However, we have a faster vanishing rate for the specific optimization paths that follow Definition \ref{def:stationary-constrained} below. Therefore, these paths better approximate true stationary points:
\begin{definition}[First-order optimal for $\rho \to \infty$]
	
	A sequence $\tbt(t)$ is  first-order optimal for  $\min_{\|\bt\|\le 1} \mathcal{L}(\rho\bt) $ with  \mbox{$\rho \to \infty$} if 
	
	\begin{inparaitem}
		\item $\lim\limits_{t \to \infty}-\frac{ \nabla_{\bt} \mathcal{L}(\rho\tbt(t))}{\|\nabla_{\bt} \mathcal{L}(\rho\tbt(t)\|}=\lim\limits_{t \to \infty} \frac{\tbt(t)}{\|\tbt(t)\|}$\,, \qquad
		\item $\|\tbt(t)\| = 1$.
	\end{inparaitem}
	
	\label{def:stationary-constrained}
\end{definition}

To relate the limit points of gradient decent to the constrained path, we will focus on stationary points of the constrained path that minimize the loss.

\begin{theorem}\label{thm:constrained-opt-homogeneous}
	Let $\bar \bt = \lim\limits_{t\to\infty} \frac{\bt(t)}{\| \bt(t)\|}$ be the limit direction of gradient descent. Under Assumptions \ref{assumption: L^* is strictly decreasing}, \ref{assu:diff}, \ref{assu:asymptotic-opt}, and constraint qualification at $\bar \bt$ (Assumption \ref{assu:cq}), the sequence $\bt(t) / \norm{\bt(t)}$ is a first-order optimal point for $\rho \to \infty$ (Definition \ref{def:stationary-constrained}). 
\end{theorem}
The proof of Theorem \ref{thm:constrained-opt-homogeneous} can be found in Appendix \ref{sec: constrained-opt-homogeneous proof}.

\subsection{Lexicographic Max-Margin} \label{sec: Lexicographic max margin}
Recall that for positive homogeneous prediction functions, the margin path $\Theta_m(\rho)$ in eq. \ref{eq: margin path} is the same set for any $\rho$ and is given by 
\vspace{-0.6em}
\[\Theta_m^*=\arg \max_{\bt\in\b{S}^{d-1}}\min_{n} f_n\left(\bt\right)\,. \vspace{-1em}
\]

For non-convex functions $f_n$ or non-Euclidean norms $\|.\|$, the above set need not be unique. In this case, we define the following refined set of maximum margin solution set
\begin{definition}[Lexicographic maximum margin set] The lexicographic margin set denoted by $\Theta_{m,N}^*$ is given by the following iterative definition of $\Theta_{m,k}^*$  for $k=1,2,\ldots,N$:
\vspace{-1.5em}
\begin{align*}
\Theta_{m,0}^*&=\b{S}^{d-1}\,,\\
\Theta_{m,k}^*&=\arg \max_{\bt\in\Theta_{m,k-1}^*}\left(\min_{\{n_\ell\}_{\ell=1}^{k}} \max_{\ell\in[k]}f_{n_\ell}\left(\bt\right)\right)\subseteq\Theta_{m,k-1}^*\,.
\end{align*}
\vspace{-1.8em}
\end{definition}
In the above definition, $\Theta_{m,1}^*=\Theta_{m}^*$ denotes the set of maximum margin solutions, $\Theta_{m,1}^*$ denotes the subset of $\Theta_{m,1}^*$ with second smallest margin, and so on. 

For an alternate representation of $\Theta_{m,k}^*$, we introduce the following notation: for $\bt\in \b{S}^{d-1}$, let $n_\ell^*(\bt)\in[N]$  denote the index corresponding to the $\ell^\text{th}$ smallest margin of $\bt$ as defined below by breaking ties in the $\arg\min$ arbitrarily: 
\begin{equation}        
\begin{split}
n_1^*(\bt)&=\arg\min_{n} f_n(\bt)\\
n_k^*(\bt)&=\arg\min_{n\notin \{n_{\ell}^*(\bt)\}_{l=1}^{k-1}} f_n(\bt)\quad\text{for }k\ge2. 
\end{split}
\end{equation}
\vspace{-0.5em}
Using this notation, we can rewrite   $\Theta_{m,k+1}^*$ as 
\begin{align*}
\Theta_{m,k+1}^*=\arg \max_{\bt\in\Theta_{m,k}^*}f_{n^*_{k+1}(\bt)}\left(\bt\right)\,.
\end{align*}
\vspace{-1.5em}

We also define the limit set of constrained path as follows:
\begin{definition}[Limit set of constrained path] The limit set of constrained path  is defined as follows:
\[\Theta_c^\infty=\left\{\bt:\!\!\!\!\begin{array}{l}\exists\{\rho_i,\bt_{\rho_i}\}_{i=1}^{\infty} \text{ with }  \rho_i\to\infty, \bt_{\rho_i}\in\Theta_{c}(\rho_i)\\
\text{ such that }\bt_{\rho_i}\to\bt\end{array}\!\!\!\!\right\}.\]
\end{definition}
\begin{theorem} \label{thm: Lexicographic max margin}
For $\hp$-positive homogeneous prediction functions the limit set of constrained path is contained in the lexicographic maximum margin set, \ie $\Theta_c^\infty\subseteq\Theta_{m,N}^*$.  
\end{theorem}

The proof of the above Theorem follows from adapting the arguments of \cite{rosset2004boosting} (Theorem $7$ in Appendix $B.2$) for general homogeneous models. We show the complete proof in Appendix \ref{app:lexmargin}.

\section{Summary}
In this paper we characterized the connections between the constrained, margin and optimization paths. First, in Section \ref{sec: Non-Homogeneous Models}, we examined general non-homogeneous models. We showed that the margin of the constrained path solution converges to the maximum margin. We further analyzed this result and demonstrated how it implies convergence in parameters, \ie $\Theta_c\left(\rho\right)$ converges to $\Theta_m\left(\rho\right)$, for some models. Then, we examined functions that are a finite sum of positively homogeneous functions. These prediction function can represent an ensemble of neural networks with positive homogeneous activation functions. For this model, we characterized the asymptotic constrained path and margin path solution. This implies a surprising result: ensembles of neural networks will aim to discard the most shallow network. In the future work we aim to analyze sum of homogeneous functions with shared variables, such as ResNets.

Second, in Section \ref{sec:homogeneous} we focus on homogeneous models. For such models we link the optimization path to the margin and constrained paths. Particularly, we show that the optimization path converges to stationary points of the constrained path and margin path. In future work, we aim to extend this to non-homogeneous models. In addition, we give a more refined characterization of the constrained path limit. It will be interesting to find whether this characterization be further refined to answer whether the weighting of the data point can have any effect on the selection of the asymptotic solution --- as \cite{byrd2018weighted} observed empirically that it did not. 


\section*{Acknowledgements}

The authors are grateful to C. Zeno, and N. Merlis for
helpful comments on the manuscript. This research was supported by the Israel Science foundation (grant No. 31/1031), and by the Taub foundation. SG and NS were partially supported by NSF awards IIS-1302662 and IIS-1764032.

\bibliographystyle{icml2019}

\newpage
	\onecolumn
	\appendix
	\part*{Appendix}
	
\section{Proof of Lemma \ref{lem:pareto-equiv}\label{sec: pareto-equiv proof}}
\begin{proof}
Let $\mathbf{w}^{*}\left(\rho\right)$ be a solution of the optimization
problem in eq. \ref{eq: min f st g}. Then, $g\left(\mathbf{w}^{*}\left(\rho\right)\right)=\rho$,
since otherwise we could have decreased $\rho$ without changing $\mathbf{w}^{*}\left(\rho\right)$
or $\phi(\rho)$ --- and this is impossible, since $\phi(\rho)$
is strictly monotonically decreasing. Therefore, we cannot decrease
$g\left(\mathbf{w}\right)$ below $\rho$ without increasing $f(\mathbf{w})$
above $\phi(\rho)$. This implies that $\mathbf{w}^{*}\left(\rho\right)$
is a solution of the optimization problem in eq. \ref{eq: min g st f}
with $\phi\left(\rho\right)$. Next, all that is left to show that
eq. \ref{eq: min g st f} has no additional solutions. Suppose by
contradiction there were such solutions $\mathbf{w}^{\prime}\left(\rho\right)$.
Since they are also minimizers of eq. \ref{eq: min g st f}, like
$\mathbf{w}^{*}\left(\rho\right)$, they have the same minimum value
$g\left(\mathbf{w}^{\prime}\left(\rho\right)\right)=\rho$. Since
they are not solutions of eq. \ref{eq: min f st g}, we have $f(\mathbf{w})>\phi(\rho)$.
However, this means they are not feasible for eq. \ref{eq: min g st f},
and therefore cannot be solutions.
\end{proof}

\section{Proof of Claim \ref{claim: aux claim 1}\label{sec: aux proof 1}}
\begin{proof}
Recall we denoted the set of solutions of eq. \ref{eq: finite gama} as $\mathcal{W}\left(\gamma^*\left(\rho\right)\right)$, and recall $\mathcal{W}$ from eq. \ref{eq: infinite gamma}. To simplify notations we omit the dependency on $\rho$ from the notation, \ie we replace $\gamma^*\left(\rho\right)$ with $\gamma$.
Suppose the claim was not correct. Then,
there would have existed $\epsilon>0$ such that $\forall\gamma$,
$\exists\gamma^{\prime}>\gamma$ such that $\exists \w^{*}\left(\gamma^{\prime}\right)\in\mathcal{W}\left(\gamma^{\prime}\right)\setminus\mathcal{B}_{\epsilon}\left(\mathcal{W}\right).$
Note that $\w^{*}\left(\gamma^{\prime}\right)\in\mathcal{W}\left(\gamma^{\prime}\right)$ is feasible in both optimization problems (eq. \ref{eq: infinite gamma}
and \ref{eq: finite gama}), since both problems have the same constraints.
Moreover, since $\w^{*}\left(\gamma^{\prime}\right)\notin\mathcal{B}_{\epsilon}\left(\mathcal{W}\right)$
it must be sub-optimal in comparison to the solution of eq. \ref{eq: infinite gamma}.
Therefore, $\exists\epsilon^{\prime}>0$ such that for any $\gamma^{\prime}$, $ \left\Vert \w_{1}^{*}\left(\gamma^{\prime}\right)\right\Vert ^{2}>\min_{\w\in\mathcal{W}}\left\Vert \w_{1}\right\Vert ^{2}+\epsilon^{\prime}$. Then we can write (from eq. \ref{eq: finite gama})
\begin{equation}
\mathcal{W}(\gamma^{\prime})=\arg\min_{\w\in\mathbb{R}^d}\left[\left\Vert \w_{1}\right\Vert ^{2}+\sum_{k=2}^{K}(\gamma^{\prime})^{\frac{2}{\hp_{k}}-\frac{2}{\hp_{1}}}\left\Vert \w_{k}\right\Vert ^{2}\right]\\\mathrm{\,\,s.t.\,\,}\forall n:f_{n}\left(\w\right)\geq1\,.\label{eq: finite gamma 2}
\end{equation}
From Assumption \ref{assum: monotone gamma} we know that $\exists c>0$ such that $\text{\ensuremath{\forall\gamma}}>c$  a solution
of the margin path exists. Therefore,  $\forall\gamma\geq c$, eq. \ref{eq: margin path} is feasible. We assume, WLOG, that $c<\gamma^{\prime}$. This implies that there exist a feasible finite
solution $\tilde{\w}$ to eq. \ref{eq: finite gamma 2} which does
not depend on $\gamma^{\prime}$. Therefore, $\forall\gamma^{\prime}$, $\forall \w\in\mathcal{W}\left(\gamma^{\prime}\right)$,
and $\forall k\in\left[K\right]$ the values of $\left\Vert \w_{k}\right\Vert ^{2}$
are respectively bounded below the values of $\left\Vert \tilde{\w}_{k}\right\Vert ^{2}$,
which are independent of $\gamma^{\prime}$. This implies that if we select
$\gamma^{\prime}$ large enough, we will have $\sum_{k=2}^{K}(\gamma^{\prime})^{\frac{2}{\hp_{k}}-\frac{2}{\hp_{1}}}\left\Vert \w_{k}\right\Vert ^{2}<\epsilon^{\prime}$.
This would contradict the assumption that $\w^{*}\left(\gamma^{\prime}\right)\in\mathcal{W}\left(\gamma^{\prime}\right)$
and therefore minimizes eq. \ref{eq: finite gamma 2}.
This implies that $\forall\epsilon$, $\exists\gamma_{0}$ such that
$\forall\gamma>\gamma_{0}$, we have $\mathcal{W}\left(\gamma\right)\subset\mathcal{B}_{\epsilon}\left(\mathcal{W}\right)$,
which entails the Theorem.

\end{proof}

\section{Proof of Lemma \ref{lem:homo-sum-constrained-margin}} \label{sec:homo-sum-constrained-margin-proof}
\begin{proof}
We assume by contradiction that yet $\bt\left(\rho\right)$
does not have the form of  eqs. \ref{eq: theta_k}  and \ref{eq: infinite gamma}. Without loss of generality we can write
\begin{equation}
\rho\bt\left(\rho\right)=\mathbf{v}_{k}\left(\rho\right)\left[\gamma\left(\rho,\bt\left(\rho\right)\right)\right]^{\frac{1}{\hp_{k}}}\,.\label{eq: rho theta form}
\end{equation}
If $\mathbf{v}_{k}\left(\rho^{\prime}\right)=\w_{k}^{*}+o\left(1\right)$,
for some $\w^{*}=\left[\w_{1}^{*},\dots,\w_{K}^{*}\right]\in\mathcal{W}$.
Then we could have written, from eqs. \ref{eq: rho theta form} and
\ref{eq: gamma difference}
\[
\rho\bt\left(\rho\right)=\left(\w_{k}^{*}+o\left(1\right)\right)\left[\gamma\left(\rho,\bt\left(\rho\right)\right)\right]^{\frac{1}{\hp_{k}}}=\left(\w_{k}^{*}+o\left(1\right)\right)\left[\gamma^{*}\left(\rho\right)\right]^{\frac{1}{\hp_{k}}}
\]
which contradicts out assumption that $\rho\bt\left(\rho\right)$
does not have the form of eq. \ref{eq: infinite gamma} and \ref{eq: finite gama}.

Therefore $\exists\delta>0$, such that $\forall\rho$, $\exists\rho^{\prime}>\rho$:
$\mathbf{v}\left(\rho^{\prime}\right)\notin\mathcal{B}_{\delta}\left(\mathcal{W}\right)$.
The norm of the solution in eq. \ref{eq: rho theta form}
\[
\rho^{\prime2}=\sum_{k=1}^{K}\left\Vert \mathbf{v}_{k}\left(\rho^{\prime}\right)\right\Vert ^{2}\left[\gamma\left(\rho,\bt\left(\rho\right)\right)\right]^{\frac{2}{\hp_{k}}}\,,
\]
is equal to the norm of the solution with margin $\gamma^{*}\left(\rho^{\prime}\right)$
\[
\rho^{\prime2}=\sum_{k=1}^{K}\left\Vert \w_{k}^{*}+o\left(1\right)\right\Vert ^{2}\left[\gamma^{*}\left(\rho^{\prime}\right)\right]^{\frac{2}{\hp_{k}}}\,.
\]
Therefore, from eq. \ref{eq: gamma difference} we have
\[
\sum_{k=1}^{K}\left\Vert \w_{1}^{*}+o\left(1\right)\right\Vert ^{2}\left[\gamma\left(\rho,\bt\left(\rho\right)\right)+C\right]^{\frac{2}{\hp_{k}}}>\sum_{k=1}^{K}\left\Vert \mathbf{v}_{k}\left(\rho^{\prime}\right)\right\Vert ^{2}\left[\gamma\left(\rho,\bt\left(\rho\right)\right)\right]^{\frac{2}{\hp_{k}}}
\]
and so, dividing by $\left[\gamma\left(\rho,\bt\left(\rho\right)\right)\right]^{\frac{2}{\hp_{k}}}$
we obtain
\begin{equation}
\left\Vert \w_{1}^{*}\right\Vert ^{2}+o\left(1\right)\\>\left\Vert \mathbf{v}_{1}\left(\rho^{\prime}\right)\right\Vert ^{2}+o\left(1\right)\label{eq: w_1 inequality}
\end{equation}
However, since $\mathbf{v}\left(\rho^{\prime}\right)\notin\mathcal{B}_{\delta}\left(\mathcal{W}\right)$,
$\exists\epsilon^{\prime}>0$ such that for all $\rho^{\prime}$:
$\left\Vert \mathbf{v}_{1}\left(\rho^{\prime}\right)\right\Vert ^{2}>\left\Vert \w_{1}^{*}\right\Vert ^{2}+\epsilon^{\prime}$
plugging this into eq. \ref{eq: w_1 inequality} we obtain 
\[
o\left(1\right)>\epsilon^{\prime}+o\left(1\right)
\]
which is a contradiction. Therefore, $\mathbf{v}\left(\rho^{\prime}\right)$
converges into $\mathcal{W}$, and eq. \ref{eq: rho theta form} can
be written in the form of eqs. \ref{eq: theta_k}  and \ref{eq: infinite gamma}.
\end{proof}

\section{Examples Section: Auxiliary Results}
\subsection{Showing that margin convergence implies convergence in the parameter space for homogeneous models} \label{sec: why margin convergence implies parameters convergence in homogeneous models}
We need to show that $\max\limits_{\bt \in\mathbb{S}^{d-1}}\min\limits_{n}f_{n}\left(\bt\right) - \min\limits_{n}f_{n}\left(\bt_{c}\left(\rho\right)\right) \to 0$ implies that $\Theta_{c}\left(\rho\right)$ converges to
$\Theta_{m}\left(\rho\right)$.

We denote $g\left(\bt\right)=\min\limits_{n}f_{n}\left(\bt\right).$ This
is a continues function since $\forall n:$ $f_{n}$ is continues.
In addition, we define for some $\rho_{0}>0$
\[
A_{r}=\left\{ \bt\in\left\{ \Theta_{c}\left(\rho\right)\right\} _{\rho\ge\rho_{o}}:\left|g\left(\bt\right)-g\left(\bt_{m}\right)\right|\le r\right\} 
\]
where $\bt_{m}\in\Theta_{m}$. Using this definition we also define
$d\left(\bt\right)$ as the Euclidean distance between $\bt$ and any
point in the set $\Theta_{m}$ and $d\left(r\right)$ as the maximal
distance for $\bt\in A_{r}$:
\begin{align*}
d\left(\bt\right) & =\min_{\boldsymbol{y}\in\Theta_{m}}\left\Vert \boldsymbol{y}-\bt\right\Vert \,,\\
d\left(r\right) & =\max_{\bt\in A_{r}}d\left(\bt\right)\,.
\end{align*}
Note that the maximum in the last equation is obtained as the maximum
of a continues function over a compact set.\\
We want to show that $\Theta_{c}\left(\rho\right)$ converges to $\Theta_{m}$.
From definition, this implies that $\forall\epsilon>0$ $\exists\rho_{0}$
such that $\forall\rho>\rho_{0}$ $\Theta_{c}\left(\rho\right)\subset\mathcal{B}_{\epsilon}\left(\Theta_{m}\right)$,
\ie $\forall\bt_{c}(\rho)\in\Theta_{c}\left(\rho\right)$: $\exists\bt'\in\Theta_{m}:\left\Vert \bt_{c}(\rho)-\bt'\right\Vert <\epsilon$.\\
Assume in contradiction that this is not the case. This means that
$\exists\epsilon>0$ such that $\forall\rho_{0}:\exists\rho>\rho_{0}$ and
$\exists\bt_{c}(\rho)\in\Theta_{c}\left(\rho\right)$ so that $\forall\bt'\in\Theta_{m}:\left\Vert \bt_{c}(\rho)-\bt'\right\Vert >\epsilon$.\\
This implies that $\lim_{r\to0}d\left(r\right)\neq0$ . Using the
limit definition we get that $\exists\epsilon>0$ so that $\forall\delta>0,$ $\exists\left|r\right|<\delta$
and $d\left(r\right)>\epsilon$. Using our notations this implies
that
\[
\exists\epsilon>0:\forall\delta:\exists\bt'\in A_{r}\text{ s.t. }\left|g\left(\bt'\right)-g\left(\bt_{m}\right)\right|\le r<\delta\text{ and }d\left(\bt'\right)>\epsilon\,.
\]
Next, we build a subsequence $\left\{ \bt_{i}\right\} _{i=1}^{\infty}$
by taking a decreasing series of $\left\{ \delta_{i}\right\} _{i=1}^{\infty}$
and their associated $\bt'$ from the last equation. Since $\bt_{i}'$
are bounded, there exist a convergent subsequence $\left\{ \widetilde{\bt}_{i}\right\} _{i=1}^{\infty}.$
For this subsequence, we obtain, using $g$ continuity
\[
\lim_{i\to\infty}g\left(\widetilde{\bt}_{i}\right)=g\left(\lim_{i\to\infty}\widetilde{\bt}_{i}\right)=g\left(\bt_{m}\right)
\]
which implies that $\exists\bt_{m}^{*}\in\Theta_{m}$ so that $\lim_{i\to\infty}\widetilde{\bt}_{i}=\bt_{m}^{*}$
which contradicts the fact that $d\left(\widetilde{\bt}_{i}\right)>\epsilon>0$. $\qed$

\subsection{Auxiliary results for $f_{n}\left(\rho,\bt\right)=\log^{1+\epsilon}\left(\rho\bt^\top\zn\right)$} \label{sec: Auxiliary calculation for log^1+epsilon example}
First, we show the full derivation of $\gamma^{*}\left(\rho\right) - \gamma\left(\rho,\bt_c\left(\rho\right)\right)$.
\begin{align*}
&\gamma^{*}\left(\rho\right) - \gamma\left(\rho,\bt_c\left(\rho\right)\right)\\
&= \log^{1+\epsilon}\left(\rho\tilde{\gamma}^*\right) - \log^{1+\epsilon}\left(\rho\tilde{\gamma}\left(\bt_c\left(\rho\right)\right)\right)\\
& = \left( \log\left(\rho\right) + \log\left(\tilde{\gamma}^*\right) \right)^{1+\epsilon}
- \left( \log\left(\rho\right) + \log\left(\tilde{\gamma}\left(\bt_c\left(\rho\right)\right)\right) \right)^{1+\epsilon}
\\
& = \log^{1+\epsilon}\left(\rho\right) + (1+\epsilon)\log^\epsilon \left(\rho\right)\log\left(\tilde{\gamma}^*\right) \\
& -  \log^{1+\epsilon}\left(\rho\right) - (1+\epsilon)\log^\epsilon \left(\rho\right)\log\left(\tilde{\gamma}\left(\bt_c\left(\rho\right)\right)\right)+o\left(\log^\epsilon \left(\rho\right)\right)\\
& = (1+\epsilon)\log^\epsilon \left(\rho\right)\left(\log\left(\tilde{\gamma}^*\right) 
- \log\left(\tilde{\gamma}\left(\bt_c\left(\rho\right)\right)\right) \right) +o\left(\log^\epsilon \left(\rho\right)\right)\\
&\le N\,.
\end{align*}

Second, we need to show that $\left(\log\left(\tilde{\gamma}^*\right) 
- \log\left(\tilde{\gamma}\left(\bt_c\left(\rho\right)\right)\right) \right)\to0$ implies that $\Theta_{c}\left(\rho\right)$ converges to
$\Theta_{m}\left(\rho\right)$.

We denote $g\left(\bt\right)=\log\left(\min\limits_{n}\bt^\top\xn\right)$. The rest of the proof is identical to the proof for the homogeneous case in Appendix~\ref{sec: why margin convergence implies parameters convergence in homogeneous models}.

\section{Proofs in Section \ref{sec:homogeneous} }

\subsection{Proof of Theorem \ref{thm:margin-opt-homogeneous}}\label{sec: margin-opt-homogeneous proof}
Define $S=\{ n: f_n (\bar \bt) = \gamma^* (1)\}$, where $\gamma^* (1)$ is the optimal margin attainable by a unit norm $\bt$.

\begin{lemma}
	Under the setting of Theorem \ref{thm:margin-opt-homogeneous},
	\begin{align}
	\nabla f_n (\bt(t)) &= \nabla f_n ( g(t) \bar \bt ) + O(B g(t) ^{\hp-1}\|\boldsymbol{\delta}(t)\|).
	\end{align}
	For $n \in S$ , the second term is asymptotically negligible as a function of $t$,
	\begin{align*}
	\nabla f_n (\bt(t)) &= \nabla f_n (g(t) \bar \bt) + o( \nabla f_n ( g(t) \bar \bt ))
	\end{align*}
	\label{lem:grad-asymptotic}
\end{lemma}
\begin{proof}
	By Taylor's theorem,
	\begin{align*}
	\nabla f_n (\bt(t)) &=\nabla f_n(g(t) \bar \bt) + \int_{s=0}^{s=1} \nabla^2 f_n ( g(t) \bar \bt + s g(t) \boldsymbol{\delta}(t) ) g(t) \boldsymbol{\delta}(t) \,\textrm{d} s\,.
	\end{align*}	
	
	Let $\bar \bt_s (t) := g(t) \bar \bt + s g(t) \boldsymbol{\delta}(t)$. We bound the integrand in the second term.
	\begin{align*}
	\norm{\nabla^2 f_n (\bar \bt_s(t) )} = \nabla^2 f_n \Big( \frac{\bar \bt_s (t)}{\norm{\bar \bt_s}} \Big) \norm{\bar \bt_s (t)} ^{\hp-2}\le B \norm{\bar \bt_s (t) }^{\hp-2},\end{align*}
	where $ B= \max_{\|\bt\| \le 1} \| \nabla^2 f_n (\bt)\|<\infty$ since $\nabla^2 f_n$ is a continuous function maximized over a compact set.
	
	Thus
	\begin{align*}
	\nabla f_n ( \bt (t) )& = \nabla f_n ( g(t)  \bar \bt ) + O( B \| \bar \theta_s (t)\|^{\hp-1}  \|\boldsymbol{\delta}(t)\|)\\
	&=\nabla f_n ( g(t)  \bar \bt )+ O(B g(t)^\hp ( 1+ o(1))^{\hp-1} \|\boldsymbol{\delta}(t)\|)\\
	&= \nabla f_n (g(t) \bar \bt) + O(B g(t) ^{\hp-1}  \|\boldsymbol{\delta}(t)\|) \tag{$\hp$ is a constant independent of $t$.}
	\end{align*}
	$\nabla f_n(g(t) \bar \bt) = g(t) ^{\hp-1} \nabla f_n ( \bar \bt)$, and for $n \in S$, $\|\nabla f_n  (\bar \bt )\| >0$ via constraint qualification (Assumption \ref{assu:cq}). Thus for $n \in S$ and using $\|\boldsymbol{\delta}(t) \| =o(1)$,
	\begin{align*}
	\nabla f_n (\bt(t)) = \nabla f_n (g(t) \bar \bt) + o(\nabla f_n (g(t) \bar \bt) ). 
	\end{align*}
\end{proof}

\begin{lemma}
	Let $S=\{ n: f_n (\bar \bt ) = \gamma^* (1)\}$.	Under the conditions of Theorem \ref{thm:margin-opt-homogeneous}, $a_n =0$ for $n \notin S$.
	\label{lem:an-0}
\end{lemma}
\begin{proof}
	\begin{align*}
	\ell_{n} (t)=\exp(-f_n (\bt(t))) = \exp(-g(t)^\hp  f_n (\bar \bt)) \exp( -g(t) ^\hp \nabla f_n (\bar \bt)^T \delta(t)) \exp(-g(t) ^\hp o(\delta(t))).
	\end{align*}
	On the other hand,
	$	\ell_n (t) = h(t)  a_n +h(t) \epsilon_n (t)$, so $\frac{\ell_n(t) }{\norm{\ell(t) }_1} \to a_n$.
	
	Consider $n \notin S$ so $f_n(\bar \bt) =\gamma_n>\gamma^* (1)$.
	\begin{align*}
	\frac{\ell_n (t)}{\norm{\ell(t)}_1}&\le \frac{\exp(-g(t) ^\hp (\gamma_n-\epsilon))} {\exp(-g(t) ^\alpha (\gamma^*(1)+\epsilon)}\\
	& \to 0 \tag{since $\gamma^*(1)+\epsilon< \gamma_n - \epsilon$ for $\epsilon$ appropriately small}
	\end{align*}
	Thus $ a_n >0 $ only if $n \in S$.
	
\end{proof}

\begin{theorem}[Theorem \ref{thm:margin-opt-homogeneous}]
	
	$\bar \bt$ satifies the first-order optimality of margin problem.
\end{theorem}
\begin{proof}
	From the gradient dynamics,
	\begin{align*}
	\bdot{\bt}(t)&= \sum_n \exp( -f_n ( \bt(t))) \nabla f_n (\bt(t)) \\
	&=\sum_n (h(t) a_n +h(t) \epsilon_{n}(t)) ( \nabla f_n (g(t) \bar \bt) + \Delta_n (t) ) ,
	\end{align*}
	where $\Delta_n (t)  = \int_{s=0}^{s=1} \nabla^2 f_n ( g(t) \bar \bt + s g(t) \delta(t) ) g(t) \delta(t)ds$.
	By multiplying out and using $a_n =0 $ for $n \notin S$ (Lemma \ref{lem:an-0}),
	\begin{align*}
	\bdot{\bt}(t)&= \underbrace{\sum_{n \in S} h(t) a_n \nabla f_n (g(t) \bt)}_{I} \\
	& \quad +\underbrace{ h(t)\sum_{n\in S}  a_n  \Delta_n (t)}_{II} +  \underbrace{h(t) \sum_{n}  \epsilon_n(t) \nabla f_n (g(t) \bt)}_{III} +\underbrace{\sum_n h(t) \epsilon_{n}(t) \Delta_n (t)}_{IV}
	\end{align*}
	Via constraint qualification (Assumption \ref{assu:cq}), $I=\Omega(  g(t)^{\hp-1} h(t)) $ and the second part of Lemma \ref{lem:grad-asymptotic}, $II = o(I)$.
	
	Since $\epsilon_{tn} =o(1)$, then $III=o(I)$.  By the first part of Lemma \ref{lem:grad-asymptotic}, $IV = O( B g(t) ^{\hp-1} \|\delta(t)\|) = o( I)$ since $\|\delta(t) \| \to 0$.
	
	Since $I$ is the largest term then after normalization,
	\begin{align}
	\frac{\bdot{\bt}(t) } {\norm{\bdot{\bt}(t)}}= \sum_{n \in S}  a_n \nabla f_n (g(t) \bar \theta)+o(1).\label{eq:lim-thetadot}
	\end{align}
	Since $\lim\limits_{t\to\infty} \frac{\bt(t) }{\|\bt(t)\|} =\lim\limits_{t\to\infty}  	\frac{\bdot{\bt}(t) } {\norm{\bdot{\bt}(t)}}$ \cite{gunasekar2018characterizing}, then 
	\begin{align}
	\lim\limits_{t\to\infty}  \frac{\bt(t) }{\|\bt(t)\|}= \sum_{n \in S} a_n \nabla f_n (g(t) \bar \bt) .\label{eq:lim-theta}
	\end{align}
	
	Thus $\bar \bt$ satisfies the first-order optimality conditions of \ref{eq:margin-homogeneous}.
\end{proof}

\subsection{Proof of Theorem~\ref{thm:constrained-opt-homogeneous}}\label{sec: constrained-opt-homogeneous proof}
\begin{proof}
	The proof is similar to the proof of Theorem \ref{thm:margin-opt-homogeneous}.
	From Equations \ref{eq:lim-theta} and \ref{eq:lim-thetadot} in the proof of Theorem \ref{thm:margin-opt-homogeneous}, we see that 
	\[\lim_{t \to \infty}-\frac{ \nabla_{\theta} \mathcal{L}(\rho\tbt(t))}{\|\nabla_{\theta} \mathcal{L}(\rho\tbt(t)\|}=\lim_{t \to \infty} \frac{\tbt(t)}{\|\tbt(t)\|}.\qedhere\]
\end{proof}

\subsection{Proof of Theorem~\ref{thm: Lexicographic max margin}}\label{app:lexmargin}
\begin{proof} The proof is  adapted from the ideas outlined in Theorem $7$ in \cite{rosset2004boosting}.

 For any $ \bt_\infty\in\Theta_c^\infty$, from definition, let $\{\rho_i,\bt_{\rho_i}\}_{i=1}^{\infty}$ denote a sequence such that  $\rho_i\to\infty$, $\bt_{\rho_i}\in\Theta_{c}(\rho_i)$ and $\bt_{\rho_i}\to\bt_\infty$. Thus, for any $\epsilon>0$, $\exists i_0$ such that $\forall i>i_0$, $\norm{\bt_\infty-\bt_{\rho_i}}\le \epsilon$. 

We need to show that  $\bt_\infty\in \Theta_{m,N}^*$.  We will prove this theorem by induction, where we show that for all $k=0,1,2,\ldots,N$, $\bt_\infty\in\Theta_{m,k}^*$. 

Recall that $\Theta_{m,0}^*=\b{S}^{d-1}$. From the definition of constrained path $\forall i:\bt_{\rho_i}\in\Theta_c(\rho_i)\subseteq \b{S}^{d-1}$. Thus, $\lim_{i\to\infty}\bt_{\rho_i}=\bt_\infty\in\b{S}^{d-1}=\Theta_{m,0}^*$, which proves the base case of induction. 

Assume that for some $k$, $\bt_\infty\in\Theta_{m,k}^*$. We need to show the inductive argument that $\bt_\infty\in\Theta_{m,k+1}^*$.

Recall that for all $\bt\in \b{S}^{d-1}$, we introduced the  notation $n_\ell^*(\bt)\in[N]$ reiterated below to denote the index corresponding to the $\ell^\text{th}$ smallest margin of $\bt$
\begin{equation}
\begin{split}
n_1^*(\bt)&=\arg\min_{n} f_n(\bt)\\
n_k^*(\bt)&=\arg\min_{n\notin \{n_{\ell}^*(\bt)\}_{l=1}^{k-1}} f_n(\bt)\quad\text{for }k\ge2,
\end{split}
\end{equation}
where in the minimization on the right, ties are broken arbitrarily. 

Using the above notation, $\Theta_{m,k+1}^*$ is given by
\begin{align*}
\Theta_{m,k+1}^*=\arg \max_{\bt\in\Theta_{m,k}^*}f_{n^*_{k+1}(\bt)}\left(\bt\right).
\end{align*}

If possible, let $\bt_\infty\notin\Theta_{m,k+1}^*$ and let $\bt'\in\Theta_{m,k+1}^*$. Using the inductive assumption and the definition of $\Theta_{m,k+1}^*$, we have we have $\bt_\infty,\bt'\in\Theta_{m,k}^*$. From the definition of $\Theta_{m,k+1}^*$, we can deduce the following,
\begin{equation}
\begin{split}
\forall \ell\le k,\; &f_{n_\ell^*(\bt_\infty)}(\bt_\infty)=f_{n_\ell^*(\bt')}(\bt')\\
&\gamma:=f_{n_{k+1}^*(\bt_\infty)}(\bt_\infty)>f_{n_{k+1}^*(\bt')}(\bt'):=\gamma'
\end{split}
\label{eq:lex-cond}
\end{equation}

Recall that $\c{L}(\bt)=\sum_n\exp(-f_n(\bt))$, where $f_n$ are $\hp$-positive homogeneous
\begin{asparaenum}[Step 1. ]
\item \emph{Upper bound on $\c{L}(\rho\bt')$}. 
\begin{align}
\nonumber\c{L}(\rho\bt')&=\sum_n\exp(-\rho^\hp f_n(\bt'))=\sum_{\ell=1}^k\exp(-\rho^\hp f_{n_\ell^*(\bt')}(\bt'))+\sum_{\ell=k+1}^{N}\exp(-\rho^\hp f_{n_\ell^*(\bt')}(\bt'))\\
&\overset{(a)}\le\sum_{\ell=1}^k\exp(-\rho^\hp f_{n_\ell^*(\bt')}(\bt'))+N\exp(-\rho^\hp \gamma')\overset{(b)}=\sum_{\ell=1}^k\exp(-\rho^\hp f_{n_\ell^*(\bt_\infty)}(\bt_\infty))+N\exp(-\rho^\hp \gamma'),
\label{eq:ub-lex}
\end{align}
where $(a)$ follows since for all $\ell>k+1$, we have $f_{n_\ell^*(\bt')}(\bt')\ge f_{n_{k+1}^*(\bt')}(\bt')=\gamma'$, and $(b)$ follows since $\forall \ell\le k$, $f_{n_{l}^*(\bt_\infty)}(\bt_\infty)=f_{n_{l}^*(\bt')}(\bt')$.
\item \emph{Lower bound on $\c{L}(\rho\bt_\infty)$}.
\begin{align}
\nonumber \c{L}(\rho\bt_\infty)&=\sum_{\ell=1}^k\exp(-\rho^\hp f_{n_\ell^*(\bt_\infty)}(\bt_\infty))+\sum_{\ell=k+1}^{N}\exp(-\rho^\hp f_{n_\ell^*(\bt_\infty)}(\bt_\infty))\\
&\overset{a}\ge\sum_{\ell=1}^k\exp(-\rho^\hp f_{n_\ell^*(\bt_\infty)}(\bt_\infty))+\exp(-\rho^\hp \gamma),
\label{eq:lb-lex}
\end{align}
where $(a)$ follows from using $f_{n_\ell^*(\bt_\infty)}(\bt_\infty)\ge0$ for all $\ell$ and using $f_{n_{k+1}^*(\bt_\infty)}(\bt_\infty)=\gamma$.
\item \emph{Lower bound on $\c{L}(\rho_i\bt_{\rho_i})$ for large enough $i$}. 
Recall that the sequence of $\bt_{\rho_i}\to\bt_\infty$ satisfies $\bt_{\rho_i}\in\Theta_c(\rho_i)$. From the definition of constrained path, we have for all $i$,  $\c{L}(\rho_i\bt_{\rho_i})\le\c{L}(\rho_i\bt_\infty)$.

We first show the following lemma:
\begin{lemma}\label{lem:lex-lb}For any $\epsilon>0$, $\exists i_0>0$ such that $\forall i\ge i_0$,  $\c{L}(\rho_i\bt_{\rho_i})\ge \sum_{\ell=1}^{k}\exp(-\rho_i^\hp f_{n_\ell^*(\bt_\infty)}(\bt_\infty))+\exp(-\rho_i^\hp(\gamma-\epsilon))$.
\end{lemma}
\begin{proof}
Recall that $\forall \ell\le k$, $\bt_{\infty}\in\Theta_{m,\ell}^*$.
Note from the definition of $n^*_{\ell}(\bt)$ that \[f_{n^*_{\ell}(\bt)}(\bt)=\min_{\{n_\ell\}_{\ell'=1}^\ell}\max_{\ell^{\prime}\in[\ell]}f_{n_{\ell^{\prime}}}.\]

Since $\bt_{\rho_i}\to\bt_\infty$, using continuity of min and max of finite number of continuous functions,  we have 
\begin{equation}
\forall\bar{\epsilon}>0,\exists i_0(\bar{\epsilon}),\text{such that }\forall i\ge i_0(\bar{\epsilon}),\forall \ell\le N,\quad\quad f_{n^*_{\ell}(\bt_{\rho_i})}(\bt_{\rho_i})\le f_{n^*_{\ell}(\bt_{\infty})}(\bt_\infty)-\bar{\epsilon}.
\label{eq:fnbound}
\end{equation}
Now consider the two cases for $i\ge i_0(\epsilon)$:
\begin{compactenum}
\item 
If $\bt_{\rho_i}\in \Theta_{m,k}^*$, then by definition, $\forall \ell\le k$, $f_{n^*_{\ell}(\bt_{\rho_i})}(\bt_{\rho_i})= f_{n^*_{\ell}(\bt_{\infty})}(\bt_\infty)$. 
Thus, 
\begin{align*}
    \c{L}(\rho_i\bt_{\rho_i})&\ge\sum_{\ell=1}^{k}\exp(-\rho_i^\hp f_{n_\ell^*(\bt_\infty)}(\bt_\infty))+ \exp(-\rho_i^\hp f_{n_{k+1}^*(\bt_{\rho_i})}(\bt_{\rho_i}))\\
    &\overset{(a)}\ge \sum_{\ell=1}^{k}\exp(-\rho_i^\hp f_{n_\ell^*(\bt_\infty)}(\bt_\infty))+ \exp(-\rho_i^\hp (\gamma-\epsilon)),
\end{align*}
where $(a)$ follows from using eq. \ref{eq:fnbound} to get  $f_{n_{k+1}^*(\bt_{\rho_i})}(\bt_{\rho_i})\le f_{n_{k+1}^*(\bt_\infty)}(\bt_\infty)-\epsilon=\gamma-\epsilon$.
\item If $\bt_{\rho_i}\notin \Theta_{m,k}^*$, let $\ell\le k$ be the smallest number such that  $\bt_{\rho_i}\notin \Theta_{m,l}^*$. So for all $\ell'<\ell$, $f_{n^*_{\ell'}(\bt_{\rho_i})}(\bt_{\rho_i})= f_{n^*_{\ell'}(\bt_{\infty})}(\bt_\infty)$, 
but $f_{n^*_{\ell}(\bt_{\infty})}(\bt_\infty)-f_{n^*_{\ell}(\bt_{\rho_i})}(\bt_{\rho_i}):=\epsilon_{i}>0$.
 
\begin{align*}
    \c{L}(\rho_i\bt_{\rho_i})&\ge\sum_{\ell'=1}^{\ell-1}\exp(-\rho_i^\hp f_{n_{\ell'}^*(\bt_\infty)}(\bt_\infty))+\exp(-\rho_i^\hp f_{n_{\ell}^*(\bt_{\rho_i})}(\bt_{\rho_i}))\\
    &=\sum_{\ell'=1}^{\ell-1}\exp(-\rho_i^\hp f_{n_{\ell'}^*(\bt_\infty)}(\bt_\infty))+\exp(\rho_i^\hp\epsilon_i)\exp(-\rho_i^\hp f_{n_{\ell}^*(\bt_{\rho_\infty})}(\bt_{\rho_\infty}))
        \end{align*}
On the other hand, 
\begin{align*}
    \c{L}(\rho_i\bt_{\infty})\le\sum_{\ell'=1}^{\ell-1}\exp(-\rho_i^\hp f_{n_{\ell'}^*(\bt_\infty)}(\bt_\infty))+N\exp(-\rho_i^\hp f_{n_{\ell}^*(\bt_{\rho_\infty})}(\bt_{\rho_\infty}))
\end{align*}
Since,  $\rho_i\to\infty$ and $\epsilon_i>0$, for large enough $i$, we have $\exp(\rho_i^\hp \epsilon_i)-N>0$
Thus, for large enough $i$, from the above two equations, we will have $\c{L}(\rho_i\bt_{\rho_i})-\c{L}(\rho_i\bt_{\infty})>0$, which is a contradiction, since $\bt_{\rho_i}\in\Theta_c(\rho_i)$. Thus, this case cannot happen for large enough $i$
\end{compactenum}
This completes the proof of the claim 
\end{proof}
\item \emph{Remaining steps in the proof.}
For any $\epsilon>0$, from eqs. \ref{eq:ub-lex}, \ref{eq:lb-lex}, and Lemma~\ref{lem:lex-lb} in Step $3$, we have the following for large  enough $i$'s
\begin{align}
    \c{L}(\rho_i\bt_{\rho_i})-\c{L}(\rho_i\bt')\ge \exp(-\rho_i^\hp(m_1-\epsilon))-N\exp(\rho_i^\hp m_2)\overset{(a)}>0.
\label{eq:lex-contradiction}
\end{align}
  where $(a)$ follows since $m_2<m_1$ and above equation holds for arbitrarily small $\epsilon$. 
 
 In eq. \ref{eq:lex-contradiction}, we have obtained a contradiction since for $\bt'\in\b{S}^{d-1}$, $\bt_{\rho_i}\in\Theta_{c}(\rho_i)\implies \c{L}(\rho_i\bt_{\rho_i})\le\c{L}(\rho_i\bt')$.
\end{asparaenum}
This completes the proof of the theorem.
\end{proof}

\section{The Regularization Path} \label{sec: Regularization Path}

\subsection{The Regularization Path}

The regularization path is given
by the following set, \mbox{$\forall c>0$}: 
\begin{equation} \label{eq: regularization path definition}
    \Theta_{r}(c)=\left\{\frac{\bt}{\norm{\bt}}:\ \arg\min_{\bt}\;\mathcal{L}\left(\bt\right)+\frac{1}{c}\left\Vert \bt\right\Vert ^{2}\right\} .
\end{equation}
\begin{lemma} \label{lem: for all c regularization path obtains optimal solution}
$\forall c>0:\ \Theta_{r}\left(c\right)$ is not empty, \ie $\forall c:\ \min_{\bt}\;\mathcal{L}\left(\bt\right)+\frac{1}{c}\left\Vert \bt\right\Vert ^{2}$ exists.
\end{lemma}
\begin{proof}
Note that $\forall c>0:\ \mathcal{L}\left(\bt\right)+\frac{1}{c}\left\Vert \bt\right\Vert ^{2}$
is coercive since $\mathcal{L}\left(\bt\right)$ is lower bounded.
Thus, the minimum of $\mathcal{L}\left(\bt\right)+\frac{1}{c}\left\Vert \bt\right\Vert ^{2}$
is attained as the minimum of a continuous coercive function over
a nonempty closed set.
\end{proof}
If Assumption \ref{assumption: L^* is strictly decreasing} is satisfied,
then as $c\to\infty$ we have that $\norm{\bt_r\left(c\right)}\to\infty$ where $\bt_r\left(c\right)\in\Theta_r\left(c\right)$. We state this result in the following lemma.

\begin{lemma} \label{lem: as c goes to infinity the solutions norm goes to infinity}
    If $\exists\rho_{0}$ such that $\mathcal{L}^{*}\left(\rho\right)$ 
    is strictly monotonically decreasing for any $\rho\geq\rho_{0}$,
    and $\bt_r\left(c\right)\in \arg\min_{\bt}\;\mathcal{L}\left(\bt\right)+\frac{1}{c}\left\Vert \bt\right\Vert ^{2}$ then as $c\to\infty$, we have $\norm{\bt_r\left(c\right)}\to\infty$.
\end{lemma}
\begin{proof}
    From lemma \ref{lem: for all c regularization path obtains optimal solution} we have that $\forall c>0$, $ \arg\min_{\bt}\;\mathcal{L}\left(\bt\right)+\frac{1}{c}\left\Vert \bt\right\Vert ^{2}$ has an optimal solution (at least one). We assume, in contradiction, that $\exists M>\rho_0$ so that $\forall c_0: \exists c>c_0$ with $\norm{\bt_r\left(c\right)}\le M$. For some $\epsilon>0$ we denote $\bt^*\in \arg\min_{\bt \in\mathbb{S}^{d-1}}\mathcal{L}\left(\left(M+\epsilon\right)\bt\right)$, \ie $\mathcal{L}\left(\left(M+\epsilon\right)\bt^*\right)=\mathcal{L}^*\left(M+\epsilon\right)$. We have that
    \begin{align*}
        &\mathcal{L}\left(\left(M+\epsilon\right)\bt^*\right)+\frac{1}{c}\norm{\left(M+\epsilon\right)\bt^*}\\
        & = \mathcal{L}\left(\bt_r\left(c\right)\right)+\frac{1}{c}\norm{\bt_r\left(c\right)}+\mathcal{L}\left(\left(M+\epsilon\right)\bt^*\right)-\mathcal{L}\left(\bt_r\left(c\right)\right)+\frac{1}{c}\norm{\left(M+\epsilon\right)\bt^*}-\frac{1}{c}\norm{\bt_r\left(c\right)}\\
        &\le \mathcal{L}\left(\bt_r\left(c\right)\right)+\frac{1}{c}\norm{\bt_r\left(c\right)}+\mathcal{L}^*\left(M+\epsilon\right)-\mathcal{L}^*\left(M\right)+\frac{1}{c}\norm{\left(M+\epsilon\right)\bt^*}
    \end{align*}
    Note that $\mathcal{L}^*\left(M+\epsilon\right)-\mathcal{L}^*\left(M\right)<0$ since we assume that $\mathcal{L}^{*}\left(\rho\right)$ 
    is strictly monotonically decreasing for any $\rho\geq\rho_{0}$. For sufficiently large $c$ we get that
    \[
        \mathcal{L}\left(\left(M+\epsilon\right)\bt^*\right)+\frac{1}{c}\norm{\left(M+\epsilon\right)\bt^*}<\mathcal{L}\left(\bt_r\left(c\right)\right)+\frac{1}{c}\norm{\bt_r\left(c\right)}
    \]
    which contradicts our assumption that $\bt_r\left(c\right)$ is an optimal solution.
\end{proof}

\subsection{Connections between regularization and constrained paths} \label{sec: Equivalence of regularization and constrained paths}
For convex loss function, the regularization and constrained paths are known to be equivalent. For general loss function, we state the following basic result.
\begin{lemma} \label{lemma: regularization and constrained paths connections}
$\forall c>0,\,\forall\theta_{r}\in\Theta_{r}\left(c\right):\,\exists\rho$ so that $\theta_{r}\in\Theta_{c}(\rho)$, 
and, If $\exists\rho_{0}$ such that $\mathcal{L}^{*}\left(\rho\right)$
is strictly monotonically decreasing for any $\rho\geq\rho_{0}$ then $\text{\, \ensuremath{\forall\theta_{r}\in\Theta_{r}\left(c\right)}:\, \ensuremath{\Theta_{c}\left(\left\Vert \theta_{r}\right\Vert \right)\subset\Theta_{r}\left(c\right)}}$. 
\end{lemma}

\begin{proof}
    To prove Lemma \ref{lemma: regularization and constrained paths connections} we combine the results from the following two lemmas.
\end{proof}

\begin{lemma}
$\forall c>0,\,\forall\theta_{r}\in\Theta_{r}\left(c\right):\,\exists\rho$ so that $\theta_{r}\in\Theta_{c}(\rho)$.
\end{lemma}
\begin{proof}
For some $c>0$, let $\bt_{r}^{*}\left(c\right)\in\Theta_{r}(c)$. From $\Theta_r(c)$ definition (eq. \ref{eq: regularization path definition})
$\left\Vert \bt_{r}^{*}\left(c\right)\right\Vert =1$ and thus $\bt_{r}^{*}\left(c\right)$ is a feasible solution of eq. \ref{eq: constrained path definition}.
Additionally, $\exists\alpha>0$ so that $\forall\bt^{(1)}$:
\[
\mathcal{L}\left(\alpha\bt_{r}^{*}\left(c\right)\right)+\frac{1}{c}\left\Vert \alpha\bt_{r}^{*}\left(c\right)\right\Vert ^{2}\le\mathcal{L}\left(\alpha\bt^{(1)}\left(c\right)\right)+\frac{1}{c}\left\Vert \alpha\bt^{(1)}\left(c\right)\right\Vert ^{2}\,.
\]

For $\rho=\alpha$ we have that $\bt_{r}^{*}\left(c\right)\in\Theta_{c}(\rho)$
since $\forall\bt^{(1)}$:
\[
\mathcal{L}\left(\rho\bt_{r}^{*}\left(c\right)\right)=\mathcal{L}\left(\rho\bt_{r}^{*}\left(c\right)\right)+\frac{\rho^{2}}{c}\left\Vert \bt_{r}^{*}\left(c\right)\right\Vert ^{2}-\frac{\rho^{2}}{c}\le\mathcal{L}\left(\rho\bt^{(1)}\left(c\right)\right)+\frac{\rho^{2}}{c}\left\Vert \bt^{(1)}\left(c\right)\right\Vert ^{2}-\frac{\rho^{2}}{c}
\]

and particularly, $\forall\bt^{(2)}$ such that $\left\Vert \bt^{(2)}\right\Vert \le1$:
\[
\mathcal{L}\left(\rho\bt_{r}^{*}\left(c\right)\right)\le\mathcal{L}\left(\rho\bt^{(2)}\left(c\right)\right)+\frac{\rho^{2}}{c}\left\Vert \bt^{(2)}\left(c\right)\right\Vert ^{2}-\frac{\rho^{2}}{c}\le\mathcal{L}\left(\rho\bt^{(2)}\left(c\right)\right)\,.
\]
\end{proof}

\begin{lemma}
$\forall c>0\text{,\, \ensuremath{\forall\bt_{r}\in\Theta_{r}\left(c\right)}:\, \ensuremath{\Theta_{c}\left(\left\Vert \bt_{r}\right\Vert \right)\subset\Theta_{r}\left(c\right)}}$. 
\end{lemma}
\begin{proof}
Note that from Lemma \ref{lem: constrained path optimal solution norm is 1}
\[
\Theta_{c}\left(\rho\right)=\arg\min_{\bt:\left\Vert \bt\right\Vert \leq1}\mathcal{L}\left(\rho\bt\right)=\left\{\frac{\bt}{\norm{\bt}}:\ \arg\min_{\bt}\mathcal{L}\left(\rho\frac{\bt}{\left\Vert \bt\right\Vert }\right)\right\}\,.
\]
For some $c>0$, let $\bt_{r}^{*}\left(c\right)\in\arg\min_{\bt}\;\mathcal{L}\left(\bt\right)+\frac{1}{c}\left\Vert \bt\right\Vert ^{2}$.
For $\rho=\left\Vert \bt_{r}^{*}\left(c\right)\right\Vert $ and
$\bt_{c}^{*}\in\arg\min_{\bt}\mathcal{L}\left(\rho\frac{\bt}{\left\Vert \bt\right\Vert }\right)$
we have that $\forall\bt$
\[
\mathcal{L}\left(\rho\frac{\bt_{c}^{*}}{\left\Vert \bt_{c}^{*}\right\Vert }\right)\le\mathcal{L}\left(\rho\frac{\bt}{\left\Vert \bt\right\Vert }\right)\,.
\]

Thus, we have that

\[
\mathcal{L}\left(\rho\frac{\bt_{c}^{*}}{\left\Vert \bt_{c}^{*}\right\Vert }\right)+\frac{1}{c}\left\Vert \bt_{c}^{*}\right\Vert ^{2}\le\mathcal{L}\left(\rho\frac{\bt}{\left\Vert \bt\right\Vert }\right)+\frac{1}{c}\left\Vert \bt_{c}^{*}\right\Vert ^{2}
\]

Thus, $\forall\bt^{(2)}$ so that $\left\Vert \bt^{(2)}\right\Vert =\rho=\left\Vert \bt_{r}^{*}\left(c\right)\right\Vert $
we have that
\[
\mathcal{L}\left(\rho\frac{\bt_{c}^{*}}{\left\Vert \bt_{c}^{*}\right\Vert }\right)+\frac{\rho^{2}}{c}\le\mathcal{L}\left(\bt^{(2)}\right)+\frac{1}{c}\left\Vert \bt^{(2)}\right\Vert ^{2}\,.
\]
In particular, this implies that
\[
    \mathcal{L}\left(\rho\frac{\bt_{c}^{*}}{\left\Vert \bt_{c}^{*}\right\Vert }\right)+\frac{\rho^{2}}{c}\le\mathcal{L}\left(\bt_{r}^{*}\left(c\right)\right)+\frac{1}{c}\left\Vert \bt_{r}^{*}\left(c\right)\right\Vert ^{2}\,.
\]
\end{proof}
\end{document}